\pdfoutput=1

\documentclass[10pt,journal,compsoc]{IEEEtran}

\usepackage{ragged2e}
\ifCLASSOPTIONcompsoc
\usepackage[nocompress]{cite}
\else
\usepackage{cite}
\fi

\ifCLASSINFOpdf
\else
\fi

\usepackage{url} 

\usepackage{times}
\usepackage{graphicx} % more modern
\usepackage{subfigure} 

\usepackage{algorithm}
\usepackage{algorithmic}

\usepackage{hyperref}

\usepackage{helvet}
\usepackage{courier}

\usepackage{makecell}
\usepackage{graphicx}
\usepackage{subfigure}
\usepackage{color}
\usepackage{amsfonts}
\usepackage{amsmath}
\usepackage{amssymb}

\usepackage{multirow}
\usepackage{booktabs}

\DeclareMathOperator*{\argmin}{arg\,min}

\def\etal{{\em et al.\/}\, }

\def\real{\text{real}}

\def\bgamma{\mbox{{\boldmath $\gamma$}}}

\def\bsigma{\mbox{{\boldmath $\sigma$}}}

\def\mB{{\mathcal B}}
\def\mC{{\mathcal C}}
\def\mD{{\mathcal D}}

\def\mF{{\mathcal F}}
\def\mG{{\mathcal G}}
\def\mH{{\mathcal H}}

\def\mM{{\mathcal M}}
\def\mN{{\mathcal N}}

\def\mU{{\mathcal U}}
\def\mV{{\mathcal V}}
\def\mW{{\mathcal W}}
\def\mX{{\mathcal X}}

\DeclareMathAlphabet\mathbfcal{OMS}{cmsy}{b}{n}
\def\0{{\bf 0}}
\def\1{{\bf 1}}

\def\bI{{\bf I}}

\def\bV{{\bf V}}

\def\bX{{\bf X}}

\def\bh{{\bf h}}

\def\br{{\bf r}}

\def\bv{{\bf v}}
\def\bw{{\bf w}}
\def\bx{{\bf x}}

\def\bz{{\bf z}}

\def\mmE{{\mathbb E}}

\def\mmR{{\mathbb R}}

\def\trsp{{\sf T}}

\def\st{{\mathrm{s.t.}}}

\def\bx{{\bf x}}
\def\bX{{\bf X}}

\def\bw{{\bf w}}

\def\bh{{\bf h}}

\def\bz{{\bf z}}

\def\st{{\mathrm{s.t.}}}

\def\ie{\mbox{\textit{i.e.}, }}
\def\eg{\mbox{\textit{e.g.}, }}
\def\wrt{\mbox{\textit{w.r.t. }}}

\usepackage{ntheorem}
\newtheorem{coll}{Corollary}
\newtheorem*{*coll}{Corollary}
\newtheorem{deftn}{Definition}
\newtheorem{thm}{Theorem}
\newtheorem*{*thm}{Theorem}
\newtheorem{prop}{Proposition}
\newtheorem{lemma}{Lemma}
\newtheorem*{*lemma}{Lemma}

\newenvironment*{proof}{\textbf{Proof}\quad}{\hfill $\square$\\ \par}

\usepackage{enumitem}

\begin{document}
	\title{Improving Generative Adversarial Networks with Local Coordinate Coding}
	
	\author{
		Jiezhang Cao$^*$,
		Yong Guo$^*$,
		Qingyao Wu,
		Chunhua Shen,
		Junzhou Huang,
		Mingkui Tan$^\dagger$
		\thanks{Jiezhang Cao and Yong Guo are with 
			School of Software Engineering, South China University of Technology, Guangzhou 510640, China, and also with Pazhou Laboratory,
			Guangzhou 510335, China. E-mail: \{secaojiezhang, guo.yong\}@mail.scut.edu.cn}
		\thanks{Mingkui Tan and Qingyao Wu are with School of Software Engineering, South China University of Technology, Guangzhou 510640, China. E-mail: \{mingkuitan, qyw\}@scut.edu.cn}
		\thanks{Chunhua Shen is 
			%SHEN
			with %the School of Computer Science, 
			The University of Adelaide, Adelaide, SA 5005, Australia. E-mail: chunhua.shen@adelaide.edu.au
		}
		\thanks{Junzhou Huang is with Tencent AI Laboratory, Shenzhen 518000, China, and also with Department of Computer Science and Engineering, University of Texas at Arlington, Arlington, TX 76019, America. E-mail: jzhuang@uta.edu%joehhuang@tencent.com
		}
		% \IEEEcompsocitemizethanks{\IEEEcompsocthanksitem M. Shell was with the Department
		% of Electrical and Computer Engineering, Georgia Institute of Technology, Atlanta,
		% GA, 30332.\protect\\
		% % note need leading \protect in front of \\ to get a newline within \thanks as
		% % \\ is fragile and will error, could use \hfil\break instead.
		% E-mail: see http://www.michaelshell.org/contact.html
		% \IEEEcompsocthanksitem J. Doe and J. Doe are with Anonymous University.}% <-this % stops an unwanted space
		\thanks{$*$ Authors contributed equally. $^\dagger$ Corresponding author.}
		%		\thanks{$^\dagger$ Corresponding author.}
		% 		\thanks{Manuscript received April 19, 2005; revised August 26, 2015.}
	}

	% \markboth{Journal of \LaTeX\ Class Files,~Vol.~14, No.~8, August~2015}%
	\markboth{Journal of \LaTeX\ Class Files, 2020}%
	{Cao \MakeLowercase{\textit{et al.}}: Improving Generative Adversarial Networks with Local Coordinate Coding}




	\IEEEtitleabstractindextext{
		\begin{abstract}
			\justifying
			Generative adversarial networks (GANs) have shown remarkable success in generating realistic data from some predefined prior distribution (\eg Gaussian noises). 
			However, such prior distribution is often independent of real data and thus may lose semantic information (\eg geometric structure or content in images) of data.
			In practice, the semantic information might be represented by some latent distribution learned from data. However,
			such latent distribution may incur difficulties in data sampling for GANs.
			In this paper, rather than sampling from the predefined prior distribution, we propose an LCCGAN model with local coordinate coding (LCC) to improve the performance of generating data. 
			First, we propose an LCC sampling method in LCCGAN to sample meaningful points from the latent manifold. With the LCC sampling method, we can  exploit the local information on the latent manifold and thus produce new data with promising quality.
			Second, we propose an improved version, namely LCCGAN++, by introducing a higher-order term in the generator approximation. This term is able to achieve better approximation and thus further improve the performance. 
			More critically,
			we derive the generalization bound for both LCCGAN and LCCGAN++ and prove that a low-dimensional input is sufficient to achieve good generalization performance.
			Extensive experiments on four benchmark datasets demonstrate the superiority of the proposed method over existing GANs.
		\end{abstract}
		
		}

	\maketitle

	\IEEEdisplaynontitleabstractindextext

	% \fi
	\IEEEpeerreviewmaketitle

	\IEEEraisesectionheading{\section{Introduction}\label{sec:introduction}}
	
	\IEEEPARstart{G}{enerative} adversarial networks (GANs) \cite{goodfellow2014gans} have been successfully applied in many computer vision tasks, such as image generation \cite{cao2018adversarial, arjovsky2017wasserstein, zhu2019dynamic, lin2019exploring, mao2018on, otverdout2020dynamic, pan2020physics}, video prediction \cite{ranzato2014video, mathieu2015deep}, image translation \cite{cao2019multi, isola2017image, kim2017learning} and domain adaptation \cite{tzeng2017adversarial, zhu2019simreal, guo2020closed}.
	In general, a GAN consists of a generator and a discriminator to play a two-player game.
	Specifically, the generator learns from a simple prior distribution (\eg Gaussian distribution \cite{goodfellow2014gans}) to produce plausible samples to fool the discriminator, while the discriminator distinguishes the fake samples from the real data.
	Recently, many studies \cite{cao2018adversarial, radford2015unsupervised, arjovsky2017wasserstein, karras2017progressive, guo2019auto} have been proposed to
	improve the performance of GANs, which, however, suffer from three limitations. 
	
	First, many GANs use some simple prior distribution, such as Gaussian distribution~\cite{goodfellow2014gans} and uniform distribution~\cite{radford2015unsupervised}.
	However, such predefined prior distribution is often independent of the data distribution.
	Besides, these methods may produce images with distorted structures without sufficient semantic information.
	Although such semantic information can be represented by some latent distributions, \eg extracting embeddings using an autoencoder~\cite{hinton2006reducing}, how to conduct sampling from these distributions still remains 
	a largely unsolved problem in GANs.
	
	Second, the correspondence between the semantic information and the dimension of the latent distribution is not yet fully exploited. 
	Most GANs~\cite{goodfellow2014gans, arjovsky2017wasserstein} use a global coordinate system to represent the data manifold and employ random noises as codings to generate data (See Fig. \ref{Fig: comparison_GCC_LCC}). 
	However, these methods fail to exploit the underlying geometry and capture the local information of data. 
	As a result, they are possible to sample meaningless points in such global coordinate system.  
	For	this issue,	how to exploit the semantic information of data and such correspondence is a very challenging problem.
	
	Third, the generalization ability of GANs \wrt the dimension of the latent distribution remains unclear.
	In practice, the performance of GANs is often sensitive to the dimension of the latent distribution~\cite{cao2018adversarial}.
	Unfortunately, it is hard to define the generalization of GANs and analyze the dimensionality of the latent distribution, since the prior distribution is independent of real data.
	Therefore, how to study the role of the dimension of the latent distribution and investigate its impact on the generalization ability become increasingly important. 
	
	In this paper, relying on the manifold assumption on images~\cite{tenenbaum2000global, roweis2000nonlinear}, we propose a novel generative model using local coordinate coding (LCC)~\cite{yu2009nonlinear} to improve the performance of GANs. 
	Specifically, we first employ an autoencoder to learn the embeddings lying on the latent manifold to capture the semantic information of real data.
	Then, we develop a new LCC sampling method for training GANs by exploiting the local information on the latent manifold. 
	For convenience, we term this method LCCGAN, which appeared in \cite{cao2018adversarial}.
	
	Based on LCCGAN, we propose an improved version, namely LCCGAN++, by introducing a higher-order term to further improve the approximation of a generator. 
	By using this term, the improved version shows more stable training behavior and is able to achieve better performance than LCCGAN.
	More critically, we analyze the generalization performance for both LCCGAN and LCCGAN++, and theoretically prove that a low-dimensional input is sufficient to achieve good generalization performance.
	
	The contributions of this paper are summarized as follows.
	\begin{itemize}%[leftmargin=*]
		\item We propose an LCC sampling method for GANs to capture the local information of real data. With the  LCC sampling method, the proposed scheme, namely LCCGAN, is able to sample meaningful points from the latent manifold to generate new data.
		
		\item Based on LCCGAN, we propose an improved version LCCGAN++ by introducing a higher-order term to further improve the approximation of a generative model. LCCGAN++ shows more stable training behavior and better performance than our preliminary work LCCGAN.
		
		\item 
		We derive the generalization bound for both LCCGAN and LCCGAN++ based on Rademacher complexity of the discriminator set and the error \wrt the intrinsic dimensionality of the latent manifold.
		In particular, we theoretically prove that a 
		low-dimensional	input is sufficient to achieve good generalization performance.
		
		\item
		Extensive experiments on several real-world datasets demonstrate the superiority of the proposed method over several baseline methods.
		Moreover, our proposed method has good scalability to generate high-resolution images even when the input dimension is low. 		
	\end{itemize}

	\section{Related Work} \label{sec:related_work}
	
	\textbf{Generative adversarial networks.} 
	Most generative adversarial networks (GANs) employ a global coordinate system with some prior distribution (such as Gaussian distribution or uniform distribution) to generate samples~\cite{goodfellow2014gans,gulrajani2017improved,karras2017progressive}.
	Unfortunately, using the global coordinate system may fail to learn the underlying geometry of data and thus often samples meaningless points to generate distorted data.
	Moreover, such prior distributions are independent of the data distributions, which may lose semantic information of real data and lead to difficulties in analyzing the dimension of latent space.
	To address this, LGAN \cite{qi2018global} uses local coordinate systems and presents a local generator whose input is sampled from a mixture of Gaussian noises with {the} discrete distribution. 
	As a result, LGAN is able to generate images 
	of good quality.
	However, this method is difficult to explore the correlation between the semantic information of real data and the dimension of a latent distribution. 
	Recently, LCCGAN \cite{cao2018adversarial} has employed a local coordinate system to exploit such correlation and improved the performance of GANs.
	
	Furthermore, some generative models conduct sampling via some learned posterior distribution.
	For example, the variational autoencoder (VAE) \cite{kingma2013auto} combines a generative model and an approximate inference model to perform posterior inference. 
	Moreover, the Wasserstein autoencoder (WAE) \cite{tolstikhin2018wasserstein} builds a real data distribution by minimizing a term of the Wasserstein distance between the model distribution and the target distribution, encouraging the encoded training distribution to match the prior. 
	In addition, the adversarial autoencoder (AAE) ~\cite{makhzani2015adversarial} matches the aggregated posterior distribution to the prior distribution to perform variational inference. 
	To further improve the performance of GANs, many methods seek to use neural architecture search techniques~\cite{zoph2016neural,guo2019nat} to automatically find good GAN models~\cite{gong2019autogan}.
	However, these methods cannot directly conduct sampling on the posterior distribution. Moreover, since they globally parameterize the manifold, they would lose local semantic information or have difficulty accessing the local geometry along the manifold.

	\textbf{Generalization analysis of GANs.}
	% 	Recently, 
	Existing methods seek to improve the generalization performance of GANs.
	Recently, Dziugaite \etal \cite{dziugaite2015training} apply maximum mean discrepancy to improve the performance of generative models, and provide a generalization analysis of the models.
	Moreover, Thanh-Tung \etal \cite{thanh-tung2018improving} show that discriminators trained on discrete datasets with the original GAN loss would fail to guarantee good generalization performance of GANs and thus provide a zero-centered gradient penalty to improve the generalization of the discriminator.
	In addition, Jiang \etal \cite{jiang2018on} derive a generalization bound under spectrum control based on the PAC-learning framework and prove that the spectrum control is able to improve the generalization ability of GAN models.
	However, these generalization analysis methods do not understand the generalization performance of GANs well from the rigorous mathematical definition.
	
	To address this shortcoming, Arora \etal \cite{arora2017gans} formally provide a definition of the generalization for GANs, and prove that the neural net distance is able to guarantee the generalization performance of GANs. 
	In contrast, the Jensen-Shannon divergence and the Wasserstein distance do not generalize with any polynomial number of examples. 
	Based on the definition of the generalization, Zhang \etal \cite{zhang2018on} use different evaluation metrics to develop several generalization bounds between the true distribution and learned distribution, and prove that the set of discriminators should be large enough to identify the true distribution and small enough surpass memorizing samples.
	Furthermore, Cao \etal \cite{cao2018adversarial} employ the neural net distance to define the generalization \wrt the dimension of the latent distribution. In addition, they develop a generalization bound related to the Rademacher complexity of the discriminator set, and prove that a low-dimensional input is sufficient to achieve good generalization performance.
	Recently, Cao \etal \cite{cao2019multi} extend the definition of the generalization of GANs to the case of multiple domains.
	However, this method is hard to understand the generalization performance of the GANs \wrt the dimension of the latent distribution.
	To better understand the generalization performance of GANs, we further study the relationship between the generalization and the dimension of the latent distribution in this paper.
	
	\begin{figure}[t]
		\centering
		\subfigure[{Global coordinate system}]
		{
			% 			\label{fig:gcc}
			\includegraphics[width = 0.47\columnwidth]{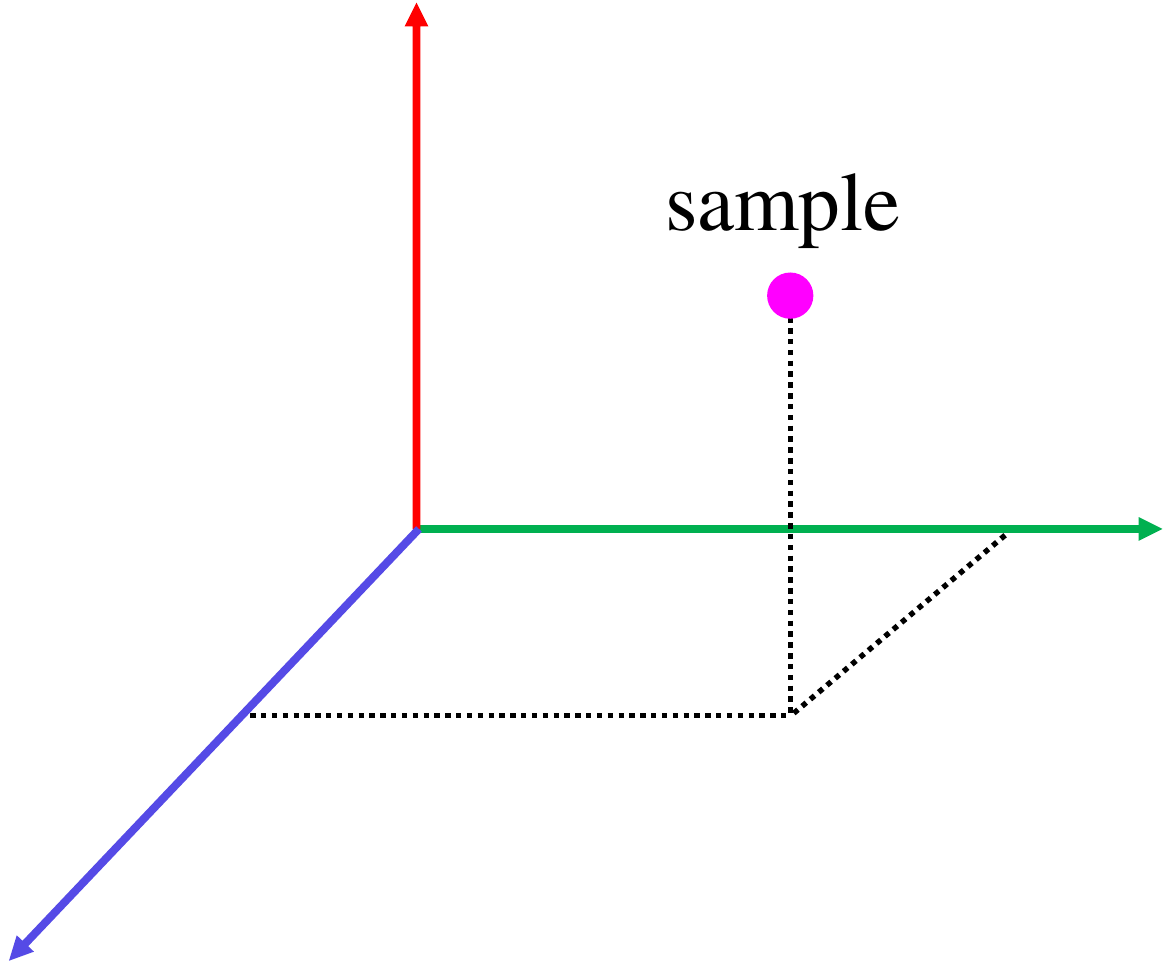}
		}
		\subfigure[Local coordinate system]{
			% 			\label{fig:lcc}
			\includegraphics[width = 0.47\columnwidth]{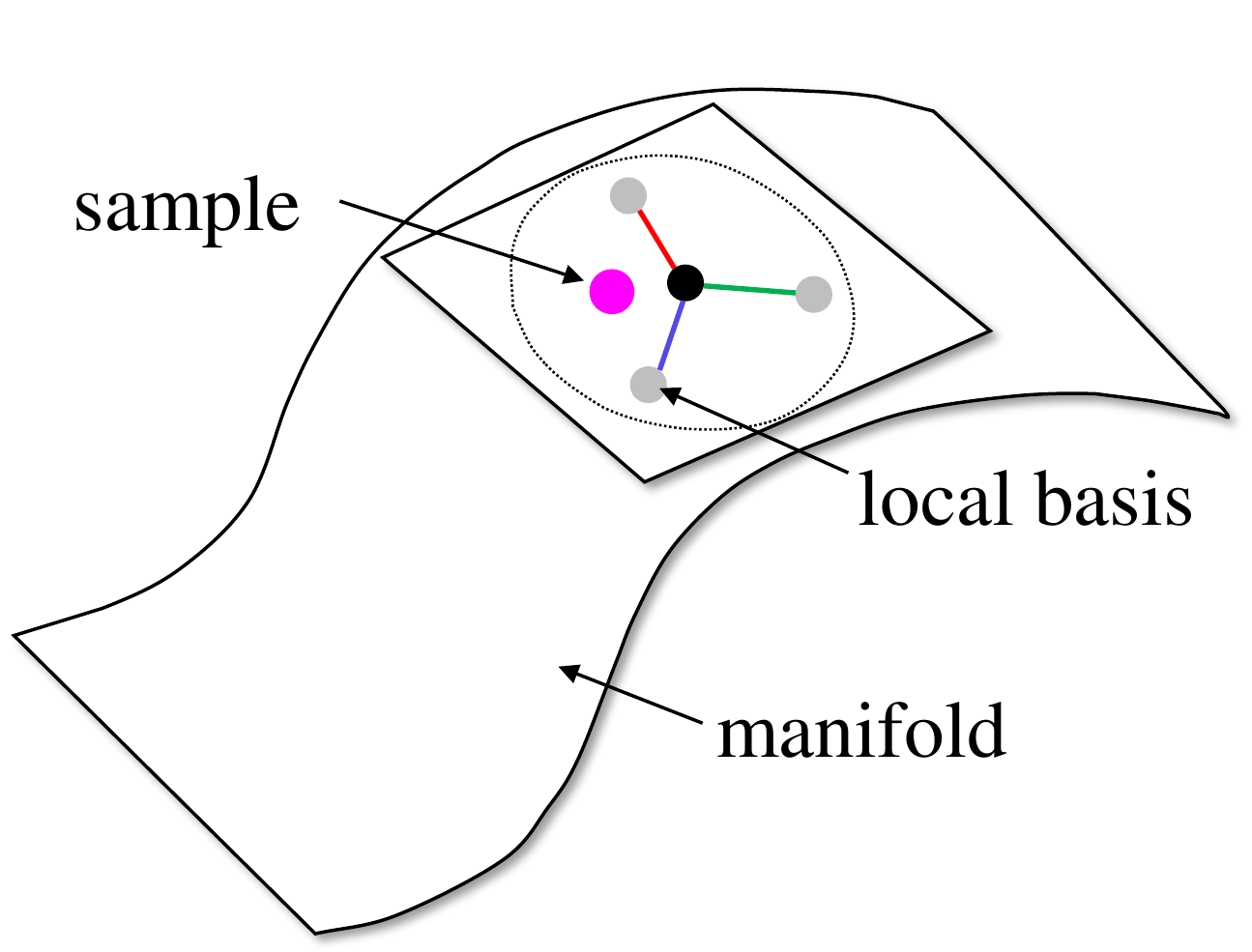}
		}
		\caption{Comparisons of the global and local coordinate system. 
			(a) In the global coordinate system, most GANs use a global cooridinate coding as an input to generate data.
			In this way, it is hard to learn the underlying geometry of real data. Therefore, they often sample meaningless points in such a global coordinate system.
			(b) In the local coordinate system, our LCCGANs learn a set of local bases on the manifold to sample new points. Then, they are able to learn the underlying geometry and capture the local information of real data. As a result, they can sample a new point with the semantic information.}
		\label{Fig: comparison_GCC_LCC}
	\end{figure}

	\section{Preliminaries} \label{sec:preliminaries}
	\textbf{Notation.}
	Throughout the paper, we use the following notations.
	Specifically, we use bold lower-case letters (\eg $ \bx $) to denote vectors and bold upper-case letters (\eg $ \bX $) to denote matrices, and we use calligraphic letters (\eg $\mX$) for a set or a space.
	Let $\mu$ and $\nu$ be distributions.
	We denote by the superscript $ ^{\trsp} $ the transpose of a vector or matrix, 
	and denote by $ \| \cdot \| $ the Euclidean norm ($ \ell_2 $-norm) on $ \mmR^d $, \ie 
	$ \| \bx \| {=} \| \bx \|_2 {=} (\sum_i x_i^2)^{1/2} $.

	\subsection{Local Coordinate Coding}
	Based on the manifold assumption on images~\cite{tenenbaum2000global, roweis2000nonlinear}, each data point $\bx$ on the manifold can be locally approximated by a linear combination of its nearby bases, and the linear weights become its local coordinate coding (LCC) \cite{yu2009nonlinear}.
	% First, 
	Specifically,
	the coordinate coding can be defined as follows.
	% 	We first introduce the following definitions of Lipschitz smoothness and local coordinate coding (See Fig. \ref{fig:LCC}). 
	% 	and then use them to develop our proposed GAN method.

	\begin{deftn} \textbf{\emph{(Coordinate coding \cite{yu2009nonlinear})}} \label{definition: Coordinate Coding}
		A coordinate coding is a pair $ (\bgamma, \mC) $, where $ \mC \subset \mmR^d $ is a set of anchor points (\ie bases), 
		and let $ \gamma $ be a map of $ \bx \in \mmR^d $ to $ \left[ \gamma_{\bv} (\bx) \right]_{\bv \in \mC} \in \mmR^{|\mC|} $ such that $ \sum_{\bv} \gamma_{\bv} (\bx) = 1 $, and the linear approximation of $ \bx \in \mmR^d $ is defined as
		\begin{align}
			\br(\bx) := \sum_{\bv \in \mC} \gamma_{\bv} (\bx) \bv. 
		\end{align}
	\end{deftn}
	
	When a data point lies on a manifold, and the bases are sufficiently localized, such data point can be approximated by a linear combination of the anchor points \cite{yu2009nonlinear}.
	In practice, such anchor points (\ie local bases) form a local coordinate system to approximate data points.
	% 	any point can be approximated by a linear combination of a set of anchor points \cite{yu2009nonlinear}.
	% 	This definition is important for local coordinate coding.
	
	In addition, we employ some useful properties (\eg Lipschitz smoothness) of a function to develop our method when the data points are in a local region. Specifically, the Lipschitz smoothness of a function can be defined as follows.
	\begin{deftn} \textbf{\emph{(Lipschitz smoothness \cite{yu2010improved})}} \label{def:Lip_smooth}
		A function $ f(\bx) $ in $ \mmR^d $ is $ (L_{\bx}, L_{f}, L_{\nu}) $-Lipschitz smooth if
		\begin{enumerate}[leftmargin=*]
			\item 
			$ \| \nabla f(\bx)^{\trsp} (\bx' - \bx) \| {\leq} L_{\bx} \| \bx {-} \bx' \| $, 
			\item 
			$ \| f(\bx') {-} f(\bx) {-} \nabla f(\bx)^\trsp (\bx' {-} \bx) \| {\leq} L_{f} \| \bx {-} \bx' \|^2 $,
			\item 
			$ \| f(\bx') {-} f(\bx) {-} \frac{1}{2} \left( \nabla f(\bx') {+} \nabla f(\bx) \right)^{\trsp} \!\!(\bx' {-} \bx) \| {\leq} L_{\nu} \| \bx {-} \bx' \|^3 $,
		\end{enumerate}
		where $ L_{\bx}, L_{f}, L_{\nu} {>} 0 $.
	\end{deftn}
	
	In Definition \ref{def:Lip_smooth}, the Lipschitz constants $ L_{\bx} $, $ L_f $ and $ L_{\nu} $ are finite if the function $ f(\bx) $, the derivative $ \nabla f(\bx) $ and the Hessian of $ f(\bx) $ are Lipschitz smooth, respectively.
	These constants measure the of smoothness of $ f(\bx) $ at different levels \cite{yu2010improved}.  %, \ie when $ \| \bx {-} \bx' \| $ is small, $ L_{\bx} $ measures how well $ f(\bx) $ can be approximated by a constant function. $ L_f $ measures how well $ f(\bx) $ can be approximated by a linear function in $ \bx $, and $ L_{\nu} $ measures how well $ f(\bx) $ can be approximated by a quadratic function in $ \bx $.
	
	\begin{figure}[t]
		\centering
		\subfigure[{Local function approximation}]
		{
			\label{fig:LCC_a}
			\includegraphics[width = 0.47\columnwidth]{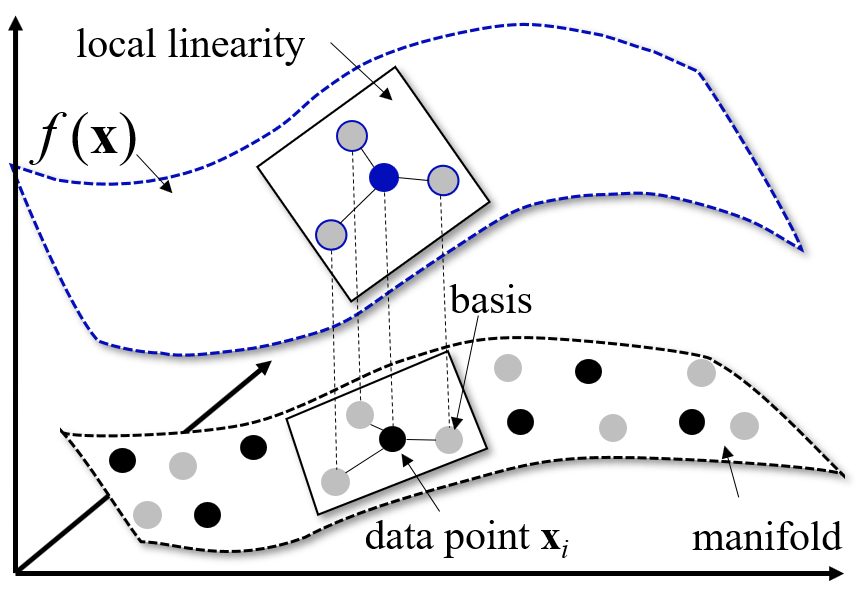}
		}
		\subfigure[Global function approximation]{
			\label{fig:LCC_b}
			\includegraphics[width = 0.47\columnwidth]{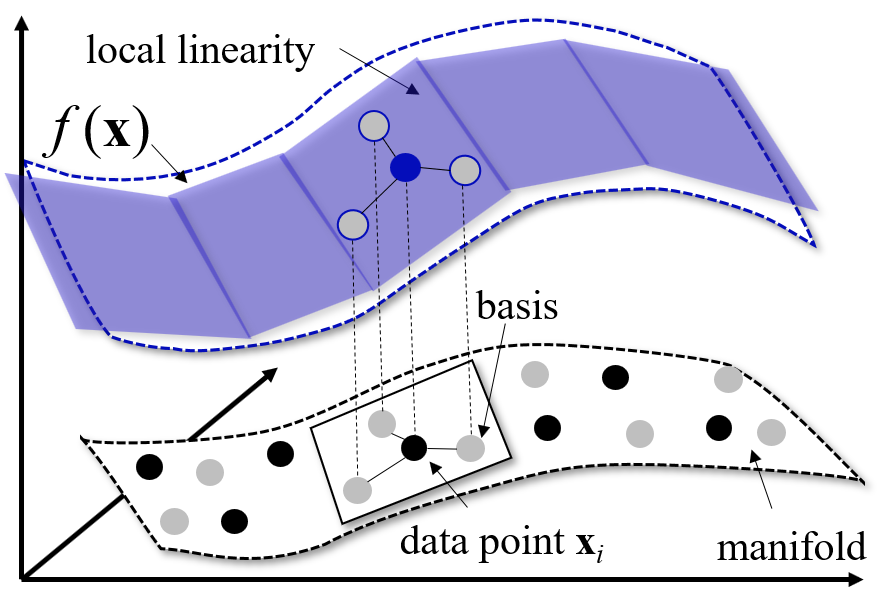}
		}
		\caption{A geometric view of the function approximation using the local coordinate coding. Given a set of local bases, if data lie on a manifold, a nonlinear function $ f(\bx) $ can be locally approximated by a linear function \wrt the local coordinate coding. Given all bases, we learn many local coordinate systems on the manifold, then the function $ f(\bx) $ can be globally approximated. }
		\label{fig:LCC}
	\end{figure}

	\subsection{Latent Manifold and Data Approximation}
	Based on the manifold assumption, high-dimensional data (\eg images) in the real world often lie on some low dimensional manifold \cite{tenenbaum2000global, roweis2000nonlinear}.
	Formally, the latent manifold and its intrinsic dimensionality can be defined as follows.
	
	\begin{deftn} \textbf{\emph{(Latent manifold \cite{yu2009nonlinear})}} \label{definition: manifold}
		A subset $ \mM $ embedded in the latent space $ \mmR^{d_{B}} $ is called a latent manifold with an \textbf{intrinsic dimension} $ d := d_{\mM} $, if there exists a constant $ c_{\mM} $, such that given any $ \bh \in \mM $, there exist $ d $ bases (tangent directions) $ \bv_1(\bh), \ldots, \bv_{d}(\bh) \in \mmR^{d_B} $ such that $ \forall\; \bh' \in \mM:$
		\begin{align}
			\inf_{\bgamma \in \mmR^{d}} \left\| \bh' - \bh - \sum_{j=1}^{d} \gamma_j \bv_j(\bh) \right\| \leq c_{\mM} \| \bh' - \bh \|^{2},
		\end{align}
		where $ \bgamma = [\gamma_1, \ldots, \gamma_d]^{\trsp} $ is the local coding of a latent point $ \bh $ using the corresponding bases.
	\end{deftn}
	
	According to Definition~\ref{definition: manifold}, one can learn a latent manifold $ \mM $ embedded in the latent space $ \mmR^{d_B} $ to build a relationship between the latent distribution and the data distribution. In this sense, we are able to generate promising images by sampling new points in the latent manifold. However, how to learn a good latent manifold is still an important problem.
	% 	Based on Definition \ref{definition: manifold}, we seek to learn a latent manifold $ \mM $ embedded in the latent space $ \mmR^{d_B} $ to build a relationship between the latent distribution and the data distribution. 
	% 	To this end, one simple approach is to use some manifold learning method, such as an autoencoder ~\cite{hinton2006reducing}, to capture the semantic information of real data. 
	% 	Specifically, given $ N $ training data $ \{ \bx_i \}_{i=1}^N $, we can use an $ \mathrm{Encoder}(\cdot) $ to extract their corresponding embeddings $ \{ \bh_i \}_{i=1}^{N} $, where $ \bh_i = \mathrm{Encoder}(\bx_i), i \in 1, \ldots, N$.
	% 	Thus, we are able to model the latent distribution relying on real data.
	
	\subsection{Generative Adversarial Networks}
	Existing studies~\cite{goodfellow2014gans, arjovsky2017wasserstein} use the Jensen-Shannon divergence and Wasserstein distance to measure the similarity between two different distributions. However, these measures cannot generalize with any polynomial number of examples \cite{arora2017gans}.
	To guarantee the generalization performance of GANs, we apply the following neural network distance \cite{arora2017gans} to measure the divergence between two distributions.
	
	\begin{deftn} \textbf{\emph{(Neural network distance \cite{arora2017gans}) }} \label{definition: F_distance}
		Let $ \mF $ be a set of neural networks
		from $ \mmR^d $ to $ [0, 1] $ and $ \phi $ be a concave measure function; then, for $ D \in \mF $, the neural network distance \wrt $ \phi $ between two distributions $ \mu $ and $ \nu $ can be defined as
		\begin{equation}
		\begin{aligned}
		d_{\mF{,} \phi} (\mu{,} \nu) = {\sup\limits_{D {\in} \mF}} &\left| \mathop \mmE\limits_{\bx {\sim} \mu} \left[ \phi(D(\bx)) \right] \right. \\
		&\left.{+} \mathop \mmE\limits_{\bx {\sim} \nu} \left[ \phi(1 - {D}(\bx)) \right] \right| {-} 2 \phi \left(\frac{1}{2}\right),
		\end{aligned}
		\end{equation}
	\end{deftn}
	where $ 2 \phi({1 \mathord{\left/ {\vphantom {1 2}} \right. \kern-\nulldelimiterspace} 2}) $ is a constant with the given function $ \phi(\cdot) $. For simplicity, we omit this term in practice. %constant $ \phi_c $.
	
	\vspace{10pt}
	\textbf{Objective function of general GANs. }
	Let $ G_{u} $ be a generator and $ D_{v} $ be a discriminator, where $ u \in \mU $ and $ v \in \mV $ are their parameters, and $ \mU $ and $ \mV $ are parameter spaces.
	Based on the definition of the neural network distance, the objective function of GANs can be defined as
	\begin{align}
		\mathop {\min}\limits_{u \in \mU} \mathop {\max}\limits_{v \in \mV} \mathop\mmE_{\bx \sim \mD_{\real}} \left[ \phi (D_{v}(\bx)) \right] + \mathop\mmE_{\bx \sim \mD_{G_{u}}} \left[ \phi (1 - D_{v}(\bx)) \right],
	\end{align}
	where $ \mD_{\real} $ is the real distribution and $ \mD_{G_{u}} $ is the distribution generated by $ G_{u} $, and $ \phi: [0, 1] \rightarrow \mmR $ is any monotone function. 
	For example, when $ \phi(t) {=} \log(t) $ and $ \mF {=} \{f{:\;} \bx {\to} [0, 1] \} $, then minimizing $ d_{\mF, \phi} (\mu, \nu) $ is equivalent to the original GAN objective.
	When $ \phi(t) {=} t $, $ f {\in} \mF $ and $ f $ is $ 1 $-Lipschitz, then $ d_{\mF, \phi} (\mu, \nu) $ corresponds to the Wasserstein distance. 
	
	% 	By optimizing the objective function, the generator tries to produce realistic data to fool the discriminator, and the discriminator tries to distinguish between the generated data and real data.

	\begin{figure*}[htp]
		\centering
		{
			\includegraphics[width=0.75\linewidth]{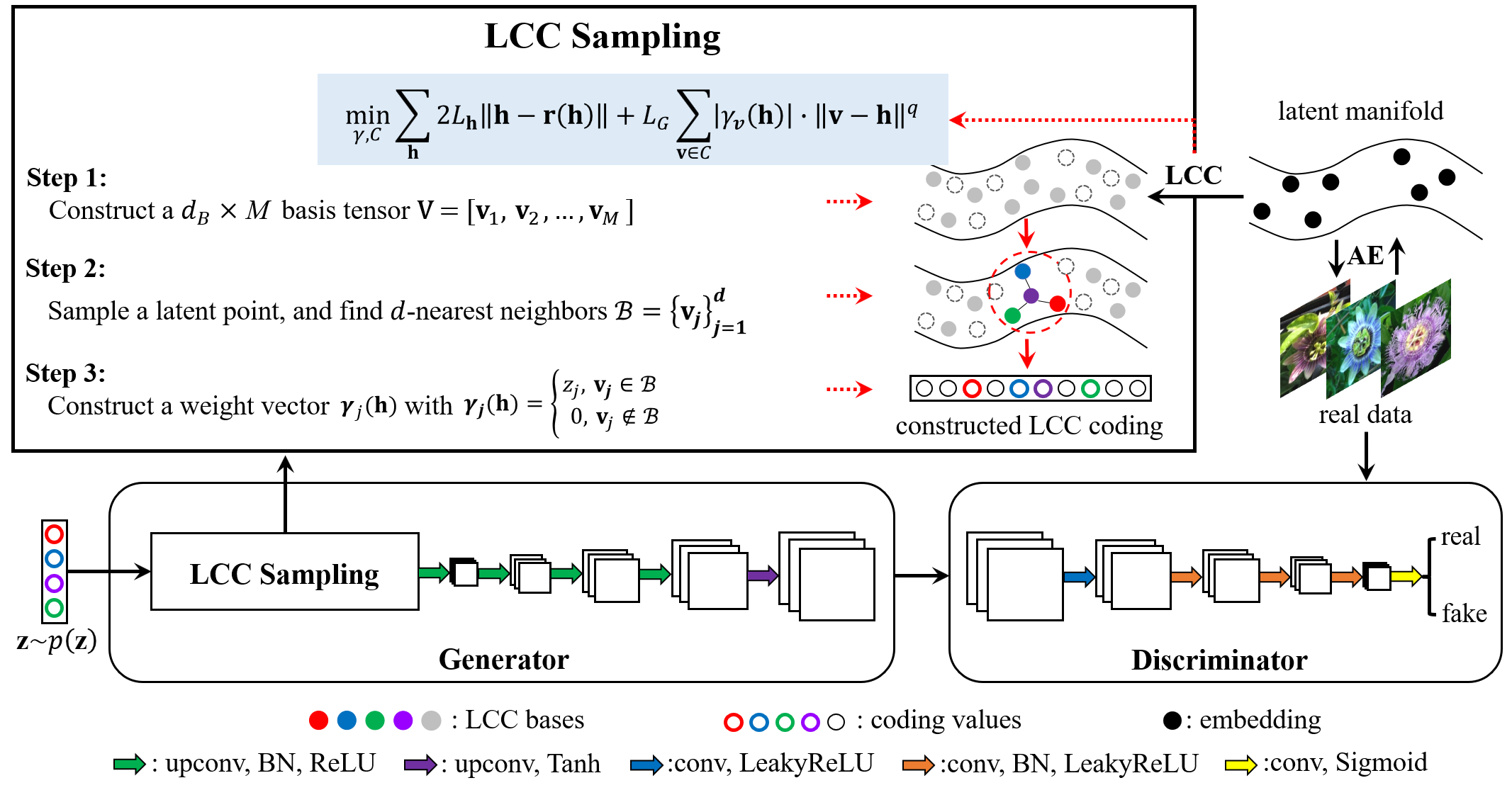}
			\caption{The scheme of the proposed method. We use an autoencoder to learn the embeddings on the latent manifold from real data.
				Then, we minimize the objective function of LCC with different $q$ to learn a set of bases such that the LCC sampling can be conducted. 
				Specifically, we train LCCGAN with $q{=}2$ and LCCGAN++ with $q{=}3$.
				Last, LCCGAN takes the constructed LCC codings as an input to generate new data. }
			\label{Fig: AE_LCC_}
		}
	\end{figure*}
	
	\section{Generative Models with LCC} \label{sec:lccgan_v1}
	Most existing methods \cite{goodfellow2014gans, arjovsky2017wasserstein} employ a global coordinate system to generate data. However, these methods often sample meaningless points in such global coordinate system. Besides, it is hard to exploit the underlying geometry and the local information of data.
	
	To address the above issues, we seek to improve GANs with LCC. The overall structure of the proposed method, called LCCGAN, is illustrated in Fig. \ref{Fig: AE_LCC_}.
	Specifically, we use an autoencoder (AE) to learn the embeddings over a latent manifold of real data and then employ LCC to learn a set of bases to form local coordinate systems. After that, we introduce 
	LCC into GANs by approximating the generator using a linear function \wrt a set of codings (See Section \ref{sebsec:G_appro}).
	Relying on such an approximation, we propose an LCC-based sampling method to exploit the local information of data (See Section \ref{subsec:lcc_sampling}). 
	% 	The details of our method are illustrated in the following subsections.
	
	\subsection{Generator Approximation Based on LCC} \label{sebsec:G_appro}
	
	Based on Definition \ref{definition: manifold}, any point on the latent manifold can be approximated by a linear combination of a set of local bases.
	Inspired by this, if the bases are sufficiently localized, the generator of GANs can also be approximated by a linear function \wrt a set of codings.
	Therefore, we approximate the generator as follows.
	
	\begin{lemma} ~\textbf{\emph{(Generator Approximation \cite{cao2018adversarial})}} \label{lemma: Generator_Approximation_1}
		Let $ \br(\bh) = \sum_{\bv} \gamma_{\bv} (\bh) \bv $,
		and $ (\bgamma, \mC) $ be an arbitrary coordinate coding.
		Given an $ (L_{\bh}, L_{G}) $-Lipschitz smooth generator $ G_{u} $, for all $ \bh \in \mmR^{d_B} $:
		\begin{equation}\label{eqn:G_appro}
		\begin{aligned}
		&\left\| G_{u}\left(\sum\nolimits_{\bv \in \mC} \gamma_{\bv}(\bh) \bv\right) {-} \sum\nolimits_{\bv \in \mC} \gamma_{\bv} (\bh) G_{u}(\bv) \right\|  \\ 
		\leq& 2L_{\bh} \| \bh {-} \br(\bh) \| {+} L_G \sum\nolimits_{\bv \in \mC} |\gamma_{\bv} (\bh)| {\cdot} \| \bv {-} \br(\bh) \|^{2}.
		\end{aligned}
		\end{equation}
	\end{lemma}
	In Lemma \ref{lemma: Generator_Approximation_1},
	given the local bases and a Lipschitz smooth generator, the generator \wrt the linear combination of the local bases can be approximated by the linear combination of the generator \wrt local bases.
	In general, two close latent points often share the same local bases but with different weights (\ie local codings), we thus can simply change these weights to approximate the generator.
	In this way, the pieces of generated data are able to cover the entire manifold seamlessly (See Fig. \ref{fig:LCC_b}).
	
	\subsection{Objective Function of LCCGAN} \label{sebsec:objective}
	
	Based on the generator approximation, we propose a learning method by exploiting LCC coding to train GAN models. Specifically, we first learn the LCC coordinate system. Then, we propose the training objective {for} the LCCGAN models. 
	
	\vspace{5pt}
	\textbf{Learning LCC systems.}
	In Step 1 of Fig. \ref{Fig: AE_LCC_}, we show an illustration of how to construct bases to form LCC systems.
	We first learn an autoencoder to extract the embeddings (\ie black points) from real data and map them to a latent manifold. Then, based on the extracted embeddings, we seek to use LCC by learning a set of bases to represent the manifold. In this way, any point located on the manifold of embeddings can be represented by the coordinate system constructed using these bases~\cite{yu2009nonlinear}.
	
	To learn the bases (\ie gray points in Fig. \ref{Fig: AE_LCC_}), we optimize the objective function of LCC, \ie we minimize the localization measure to obtain a set of local bases.
	Specifically, given a set of the latent points $ \{ \bh_i \}_{i=1}^N $, by assuming $ \bh {\approx} \br(\bh) $ \cite{yu2009nonlinear}, we seek to address the following optimization problem:
	\begin{align}
		\min_{\bgamma, \mC}& \sum\nolimits_{\bh} 2L_{\bh} \| \bh - \br(\bh) \| {+} L_G \sum\nolimits_{\bv \in \mC} |\gamma_{\bv} (\bh)| {\cdot}\| \bv - \bh\|^{2}, \nonumber\\
		\st& \sum\nolimits_{\bv \in \mC} \gamma_{\bv} (\bh) = 1, \; \forall \,\bh, \label{problem: lcc coding}
	\end{align}
	where $ \bh $ denotes an embedding learned by an autoencoder from real data, $ \mC $ denotes the set of local bases, and $ \br(\bh) = \sum_{\bv \in \mC} \gamma_{\bv} (\bh) \bv $.
	In practice, we normalize the weights $ \bgamma $ to the sum of 1 during the training, and update $ \bgamma $ and $ \mC $ by alternately optimizing a LASSO problem and a least-square regression problem, respectively.
	After optimizing Problem (\ref{problem: lcc coding}), we can construct the local bases on the latent manifold.
	
	\begin{figure}[t]
		\centering
		{
			\includegraphics[width=1\linewidth]{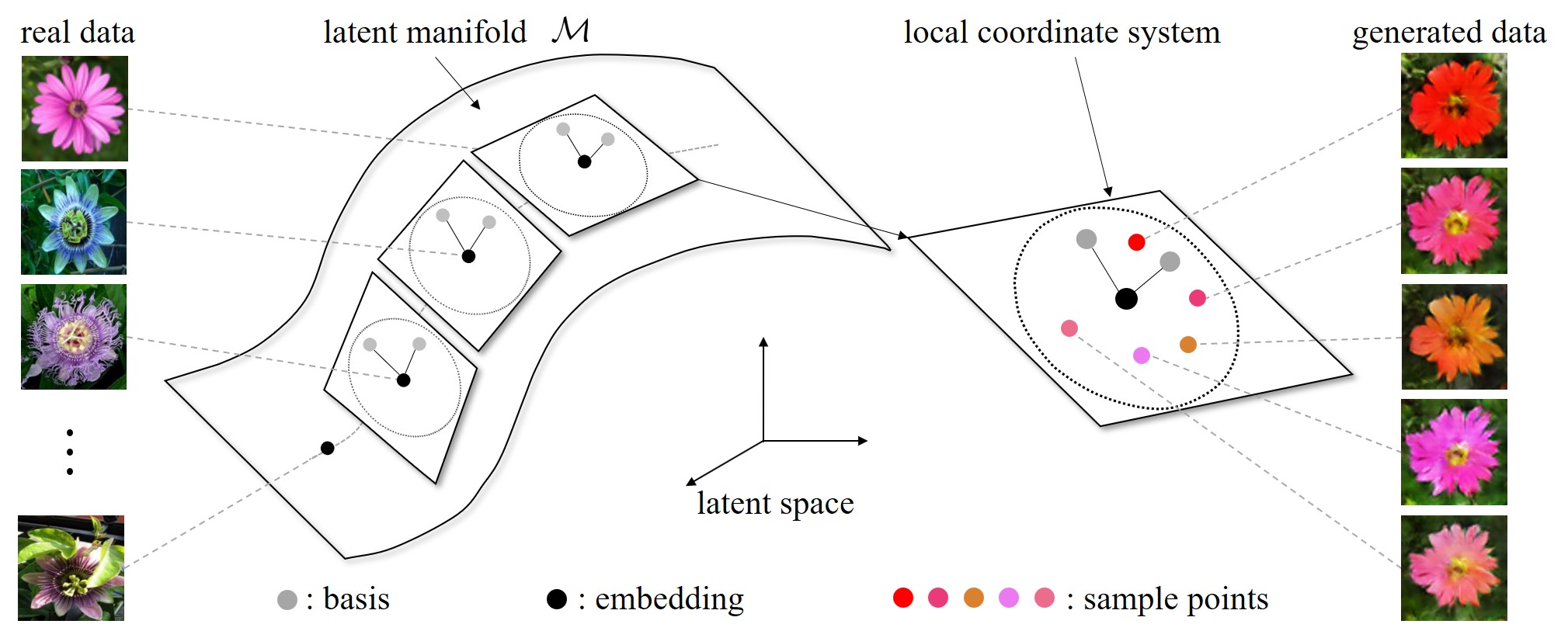}
			\caption{The geometric views on the LCC sampling method. By learning embeddings (\textit{i.e.}, black points) that lie on the latent manifold, we use LCC to learn a set of bases (\textit{i.e.}, gray points) to form a local coordinate system such that we can sample different latent points (\textit{i.e.}, colored points) by LCC sampling. As a result, our proposed method can generate new data that have different attributes. }
			\label{Fig: LCC_Sampling}
		}
	\end{figure}
	
	\vspace{5pt}
	\textbf{Training LCCGAN.}
	After solving Problem (\ref{problem: lcc coding}), every latent point $ \bh {\in} \mmR^{d_B} $ would be close to its physical approximation $ \br(\bh) $, \textit{i.e.}, $ \bh {\approx} \br(\bh) $, then the generator can be approximated by
	\begin{align}\label{eqn: generator}
		G_{u}(\bh) \approx G_{u}(\br(\bh)) \triangleq G_{w} (\bgamma(\bh)), \bh \in \mH,
	\end{align}
	where $ \br(\bh) = \bV \bgamma(\bh) $, $ \bV = \left[ \bv_1, \bv_2, \ldots, \bv_M \right] $ and $ \bgamma(\bh) = \left[ \gamma_1(\bh), \gamma_2(\bh), \ldots, \gamma_M(\bh) \right]^{\trsp} $ with $ M = |\mC| $. Here, $ \mH $ is the latent distribution and $ w {\in} \mW $ are the parameters of the generator \wrt $ u $ and fixed $ \bV $ learned from Problem (\ref{problem: lcc coding}).
	Note that the input of the generator $ G_{w} (\bgamma(\bh)) $ in this paper is local coordinate coding, which is different from other GANs. 
	
	According to Definition \ref{definition: F_distance}, we apply the neural network distance to measure the divergence between the generated distribution and the empirical distribution.
	Specifically, given the generator $ G_{w} (\bgamma(\bh)) $, we consider optimizing the following objective function for LCCGAN:
	\begin{align}\label{problem: LCCGANgan}
		\min_{G_{w} \in \mG} \; d_{\mF, \phi} \left(\widehat\mD_{G_{w} (\bgamma(\bh))}, \widehat\mD_{\real} \right), \bh \in \mH,
	\end{align}
	where $ \mG $ is the class of generators, $ \widehat\mD_{G_{w}} $ is the empirical distribution generated by $ G_{w} $, and $ \widehat\mD_{\real} $ is the real distribution.
	Specifically, Problem (\ref{problem: LCCGANgan}) can be rewritten as:
	\begin{align*}
		\mathop{\min}_{w {\in} \mW} \mathop{\max}\limits_{v {\in} \mV}
		\mathop \mmE_{\bx {\sim} \widehat{\mD}_{\real}} \left[ \phi (D_{v}(\bx)) \right]
		{+} \mathop{\mmE}_{\bh {\sim} \mH} \left[ \phi \left(1{-} {D}_v\left(G_{w}\left(\bgamma(\bh) \right) \right) \right) \right],
	\end{align*}
	where $ \phi(\cdot) $ is a monotone function. 
	Then, the objective function can be used in different GANs, such as DCGAN \cite{radford2015unsupervised} and WGAN-GP~\cite{gulrajani2017improved}.
	The detailed algorithm is shown in Algorithm \ref{alg:lccgan}.

	\subsection{LCC Sampling Method} \label{subsec:lcc_sampling}
	To solve Problem (\ref{problem: LCCGANgan}), one of the key issues is how to conduct sampling from the learned latent manifold.
	Although the latent manifold can be learned by an autoencoder, it is difficult to sample valid points on it to train GANs.
	To address this, we propose an LCC sampling method to capture the latent distribution on the learned latent manifold (See Fig.~\ref{Fig: LCC_Sampling}). 
	The proposed sampling method contains the following three steps.
	
	\vspace{5pt}
	\noindent\textbf{Step 1:}
	Given a local coordinate system, we construct an $d_{B} {\times} M$  matrix $\bV {=} [ \bv_1, \bv_2, \ldots, \bv_M] $ as the local bases. Here, each basis $\bv_i$ is a $d_{B}$-dimensional vector and $M$ is the number of bases.
	
	\vspace{3pt}
	\noindent\textbf{Step 2:}
	With the learned local bases $\bV$, we randomly sample a latent point (specifically, it can be a basis), and then find its $d$-nearest neighbors $ \mB = \{ \bv_j\}_{j=1}^d $.
	
	\begin{algorithm}[t]
		\caption{Training Method for LCCGAN.}
		\label{alg:lccgan}
		\begin{algorithmic}[1]
			\begin{small}
				\REQUIRE  Training data $\{ \bx_i \}_{i=1}^N $; a prior distribution $p(\bz)$, where $\bz \in \mmR^d  $; minibatch size $ n $; $q=2$ or $q=3$.\\
				\STATE Learn the latent manifold $\mM$ using an autoencoder \\
				\STATE Construct LCC bases $\{ \bv_i \}_{i=1}^M$ on $\mH$ by optimizing: \\
				\vspace{5pt}
				~~$ {\min_{\bgamma, \mC}}~ {\sum_{\bh}} 2L_{\bh} \| \bh - \br(\bh) \| + L_G {\sum_{\bv \in \mC}} |\gamma_{\bv} (\bh)| {\cdot} \| \bv - \bh \|^{q}$ \\
				\vspace{5pt}
				\FOR{number of training iterations}
				\STATE Do LCC Sampling to obtain a minibatch $\{ \gamma(\bh_i)\}_{i=1}^n$ \\ %, where $n$ is the size of minibatch  \\
				\STATE Sample a minibatch $\{\bx_i\}_{i=1}^n$ from the data distribution \\
				\STATE Update the discriminator by ascending the gradient: \\
				\vspace{5pt}
				~~~~~~~$\nabla_v \frac{1}{n} \sum\nolimits_{i=1}^n \phi(D_v(\bx_i)) + \phi( (1-D_v(G_w(\gamma(\bh_i)))))$\\
				\vspace{5pt}
				\STATE Do LCC Sampling to obtain a minibatch $\{ \gamma(\bh_i)\}_{i=1}^n$ \\
				\STATE Update the generator by descending the gradient: \\
				\vspace{5pt}
				~~~~~~~~~~~~~~~~~~$\nabla_w \frac{1}{n} \sum\nolimits_{i=1}^n \phi (1-D_v(G_w(\gamma(\bh_i))))$\\
				\ENDFOR
			\end{small}
		\end{algorithmic}
	\end{algorithm}
	
	\vspace{3pt}
	\noindent
	\textbf{Step 3:}
	To conduct the local sampling method, we construct an $M$-dimensional vector $\bgamma(\bh) = [ \gamma_1(\bh), \gamma_2(\bh), \ldots, \gamma_M(\bh)]^{\trsp}$
	as the LCC coding. 
	The weight $\gamma_j(\bh)$ for the $j$-th element of $\bgamma(\bh)$ can be computed as follows:
	\begin{align}
		\gamma_{j} (\bh) =
		\left\{ \begin{array}{l}
			{z_j},\;\;{\bv_j} \in {\mB}\\
			\; 0\,,\;\;{\bv_j} \notin {\mB}
		\end{array} \right.,
		\label{eqn:gamma_local_base}
	\end{align}
	where $ z_j $ is the $ j $-th element of $ \bz {\in} \mmR^d $ from the prior distribution $ p(\bz) $.
	Here, we set $ p(\bz) $ to be the Gaussian distribution $ \mN(\0, \bI) $ and normalize the sum of $ \bgamma(\bh) $ to be 1 in the training, \ie $\sum_{j} \gamma_{j} (\bh) {=} 1$.
	In this paper, we use Gaussian distribution for two reasons. 
	First, Gaussian distribution is an available way for sampling, which has been widely used in many GANs~\cite{arjovsky2017wasserstein, goodfellow2014gans}.
	In Fig. \ref{Fig: LCC_Sampling}, given the latent manifold, we employ LCC to form local coordinate systems over the latent manifold, \ie built with a set of local bases (\ie gray points).
	In the local coordinate system, we use Gaussian distribution to sample a new point $\bV \bgamma(\bh)$ (\ie colored point) by specifying the weights for the local bases. 
	In this way, we can generate images by exploiting the local information of data.
	Second, by using Gaussian distribution for sampling, it is reasonable and fair to compare LCCGAN with other GANs. 
	The advantages of LCCGAN using the local coordinate system over other GANs can be found in Section \ref{sec:experiment}.
	
	Based on Definition \ref{definition: manifold}, the intrinsic dimensionality is determined by the number of bases in a local region. Thus, we turn the determination of the intrinsic dimension into an easier problem of selecting a sufficient number of local bases.

	\subsection{Effectiveness of LCC Sampling}
	
	We first discuss the relationship between LCC and the LCC sampling method. Then, we analyze the effect of LCC in GANs. 
	
	\textbf{Relationship between LCC and LCC sampling.}
	The LCC sampling method is closely related to LCC for two reasons.
	First, both of them rely on the local coordinate system.
	In Fig.~\ref{Fig: LCC_Sampling}, we learn a set of bases (\ie gray points) to form a local coordinate system by optimizing the objective function of LCC.  
	Second, both of them can effectively exploit the local information of real data.
	Based on the learned bases, we can use the proposed LCC sampling method to sample different points (\ie colored points) in a local area of the latent manifold. 
	
	\textbf{How does LCC improve GANs?}
	When introducing LCC into a GAN model, we can use the local coordinate system to exploit the local information of data, and thus improve the performance of GANs.
	In contrast, most GANs~\cite{goodfellow2014gans, arjovsky2017wasserstein} use a global coordinate system, which, however, would fail to capture the semantic information of real data.
	In this sense, they are possible to sample meaningless points. 
	To verify this, we show the advantage of the local coordinate system over the global coordinate system, as shown in Table \ref{tab:flowers}.

	\section{Generative Models with Improved LCC}
	\label{sec:lccgan_v2}
	When learning local coordinate systems, the linear combination of the generator \wrt the local bases may be far away from the manifold. As a result, the generator may sample a meaningful point such that the image quality is poor. To address this,
	we propose an enhanced GAN, called LCCGAN++, to improve the approximation of the generator.
	In the following, we first improve the generator approximation of LCCGAN, and then analyze the generalization performance.
	% 	First, we improve the generator approximation of LCCGAN, and discuss the difference between these two methods.
	% 	Then, we analyze the generalization performance of LCCGAN++.

	\subsection{Improved Generator Approximation} \label{sebsec:extend_G_appro}
	By minimizing the right-hand side of (\ref{eqn:G_appro}), the generator equipped with LCC~\cite{cao2018adversarial} has a small approximation error. 
	However, the local linear approximation may not necessarily be optimal when the generator is highly nonlinear.
	It means that many local bases are required to achieve better approximation.
	As suggested by \cite{yu2010improved}, the higher-order error term would have a better generator approximation.
	Thus, we can improve LCC by introducing a higher-order term.
	Then, we have the corresponding generator approximation in the following lemma.
	
	\begin{lemma} ~\textbf{\emph{(Improved generator approximation)}} \label{lemma:extend_G Appro}
		Let $ \br(\bh) = \sum_{\bv} \gamma_{\bv} (\bh) \bv $.
		Given an arbitrary coordinate coding $ (\bgamma, \mC) $ and an $ (L_{\bh}, L_{\nu}) $-Lipschitz smooth generator $ G_{u} $, for all $ \bh$:
		\begin{align}\label{eqn:ext_G_appro}
			&\left\| G_{u}\left(\br(\bh)\right) {-} \sum\limits_{\bv \in \mC} \gamma_{\bv}(\bh) \left( G_{u}(\bv) {+} \frac{1}{2} \nabla G_{u}(\bv)^{\trsp} (\bh {-} \bv) \right)  \right\|  \nonumber\\ 
			\leq& 2 L_{\bh} \| \bh {-} \br(\bh) \| {+} L_{\nu} \sum_{\bv \in \mC} |\gamma_{\bv} (\bh)| {\cdot} \| \bv {-} \br(\bh) \|^{3}.
		\end{align}
	\end{lemma}
	
	In Lemma \ref{lemma:extend_G Appro}, the generator \wrt the linear combination of the local bases can be approximated by introducing gradient directions.
	% 	Based on LCC-v1, we improve LCC by using a higher order term, called LCC-v2.  
	Compared the right-hand side of (\ref{eqn:G_appro}) with (\ref{eqn:ext_G_appro}), the first term is similar and can be small when $ \bh $ can be well approximated by a linear combination of local bases, which happens when the manifold is relatively flat. 
	For the second term, the improved LCC has a higher-order term which enforces the learned bases to be close to the linear combination of the local bases.

	\subsection{Differences between LCCGAN and LCCGAN++}
	LCCGAN++ is different from LCCGAN in the following aspects. First, when the number of the local bases is insufficient, the linear combination of the generator \wrt the local bases would be far away from the manifold.
	As a result, we have the poor approximation of the generator. 
	Besides, the generated images of LCCGAN may have poor quality.
	Second, the generator \wrt the local bases can be transformed into the locally flat region approximately along the gradient of the generator.
	In this way, the linear combination of the generator \wrt the local bases would be close to the manifold.
	Therefore, with the linear combination of bases as input, we have a good generator approximation to generate realistic images.
	
	Compared with LCCGAN, our proposed LCCGAN++ mainly introduces a higher-order term to improve the approximation of the generator.
	Relying on this term, LCCGAN++ has more stable training behavior and achieves better generalization performance than LCCGAN.
	
	\begin{figure*}[h]
		\centering
		\subfigure[Generated samples with $d{=}3$. 
		The yellow and red boxes denote the similar generated digits with low diversity.
		]{
			\label{fig:mnist_low}
			\includegraphics[width = 2\columnwidth]{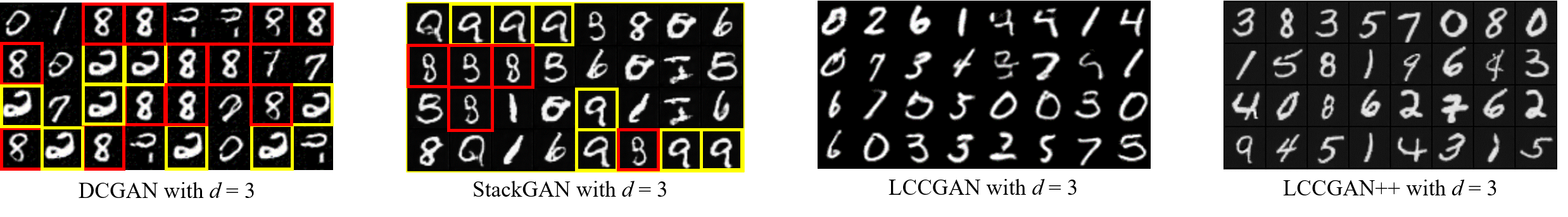}\label{fig:ratio}
		}
		\subfigure[Comparisons of different GANs with the input noise of $d{=}5$. Besides, DCGAN with $d{=}100$ is considered as the baseline.]{
			\label{fig:mnist_high}
			\includegraphics[width = 2\columnwidth]{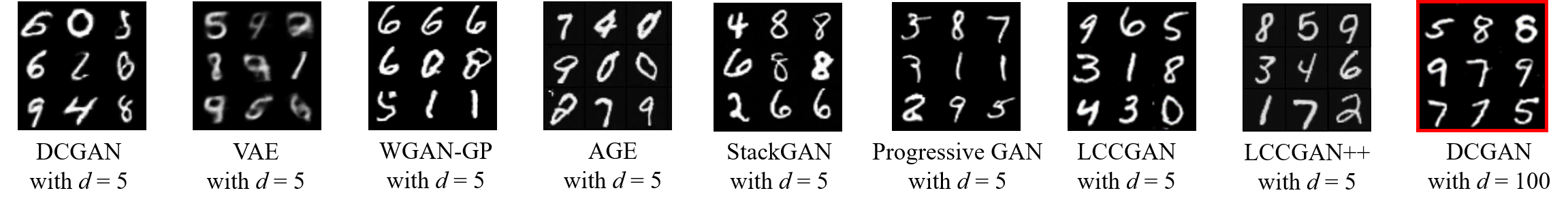}\label{fig:ratio_testing_error}
		}
		\caption{Performance comparisons of various GANs with different dimensions of the latent distribution on the MNIST dataset.}
		\label{fig:mnist}
	\end{figure*}

	\subsection{Theoretical Analysis}\label{section: LCCGAN Theoretical Analysis}
	We first provide some necessary notations. % to develop our theoretical analysis.
	Let $\{ \bx_i \}_{i=1}^N $ be a set of observed training samples drawn from the real distribution $ \mD_{\real} $, and let $ \widehat{\mD}_{\real} $ denote the empirical distribution over $\{ \bx_i \}_{i=1}^N $.  
	Let $ \mD_{G_u} $ be the generated distribution, and $ \widehat{\mD}_{G_{w}} $ be an empirical generated distribution.
	Motivated by \cite{arora2017gans, zhang2018on}, we define the generalization of GANs as follows:
	\begin{deftn} \textbf{\emph{(Generalization)}} \label{definition: generalization}
		The neural network distance $ d_{\mF, \phi}(\cdot, \cdot) $ between distributions generalizes with $ N $ training samples and error $ \epsilon $, if for a learned distribution $ \mD_{G_{u}} $, the following inequation holds with high probability,
		\begin{align}
			\left|d_{\mF, \phi} \left(\widehat{\mD}_{G_{w}}, {\mD}_{\emph{\real}} \right) - \inf\nolimits_{\mG } d_{\mF, \phi} \left({\mD}_{G_u},  {\mD}_{\emph{\real}} \right) \right| \leq \epsilon.
		\end{align}
	\end{deftn}
	
	From Definition \ref{definition: generalization}, the population distance $ d_{\mF, \phi} ({\mD}_{G_u},  {\mD}_{\real} ) $ shall be close to the distance $ d_{\mF, \phi} (\widehat{\mD}_{G_{w}}, {\mD}_{\real} ) $. In theory, we hope to obtain a small $ d_{\mF, \phi} ({\mD}_{G_u},  {\mD}_{\real} ) $ to ensure good generalization ability.  In practice, we can minimize the empirical loss $ d_{\mF, \phi} (\widehat{\mD}_{G_{{w}}}, \widehat{\mD}_{\real} ) $ to approximate  $ d_{\mF, \phi} (\widehat{\mD}_{G_{{w}}}, {\mD}_{\real} ) $.

	For LCCGAN~\cite{cao2018adversarial}, we have developed a generalization bound on $ \widehat{\mD}_{\real} $. In the following, we further analyze the generalization of LCCGAN++  relying on the improved generator approximation.

	\begin{thm} \label{theorem: Further Generalization Bound}
		Suppose that $ \phi(\cdot) $ is Lipschitz smooth, and bounded in $ [-\Delta, \Delta] $. Given an sample set $ \mH $ in the latent space and an empirical distribution $ \widehat{\mD}_{\emph{\real}} $ with $ N $ samples drawn from $ \mD_{\emph{\real}} $, the following inequation holds with probability at least $ 1 - \delta $,
		\begin{align}
			& \left| \mmE_{\mH} \left[d_{\mF, \phi} \left(\widehat{\mD}_{G_{\widehat{w}}}, {\mD}_{\emph{\real}} \right) \right]
			{-} \inf_{ \mG } \mmE_{\mH} \left[ d_{\mF, \phi} \left({\mD}_{G_u},  {\mD}_{\emph{\real}} \right) \right] \right| \nonumber\\
			\leq& 2 {R}_{\mX}(\mF) + 2 \Delta \sqrt{\frac{2}{N} \log\left(\frac{1}{\delta}\right)} + 2\epsilon(d_{\mM}),
		\end{align}
		where $ {R}_{\mX}(\mF) $ is the Rademacher complexity of $\mF$, the error term $ \epsilon(d_{\mM}) {=} L_{\phi} Q_{L_{\bh}, L_{\nu}} (\bgamma, \mC) + 2\Delta $, and $ Q_{L_{\bh}, L_{\nu}} (\bgamma, \mC) $ has an upper bound \wrt $ d_{\mM} $ which is given in Supplementary materials.
	\end{thm}
	
	% 	In Theorem \ref{theorem: Further Generalization Bound},
	The error term $ \epsilon(d_{\mM}) $ indicates that a low dimensional input is sufficient to achieve good generalization.
	% 	In practice, different datasets have different intrinsic dimensions of the latent manifold. 
	Moreover, the experiments justify that our method is able to generate perceptually convincing images with low-dimensional inputs.
	
	Note that Theorem \ref{theorem: Further Generalization Bound} is slightly different from the results of LCCGAN \cite{cao2018adversarial} because $ Q_{L_{\bh}, L_{\nu}} (\bgamma, \mC) $ is related to the high-order term.
	% 	Based on Theorem \ref{theorem: Further Generalization Bound} and the Rademacher complexity \cite{bartlett2002rademacher}, 
	Then, we consider a specific discriminator set to analyze and understand the generalization performance of LCCGAN++.
	
	\begin{coll}\label{coll:generalization_relu_D}
		Let %$ \mX $ be the unit ball of $ \mmR^d $ under the $ \ell_2 $-norm, \ie 
		$ \mX {=} \{ \bx {\in} \mmR^d{:\;} \| \bx \| {\leq} 1 \} $.
		Assume that the discriminator set $ \mF $ is the set of neural networks with a rectified linear unit, \ie
		$\mF {=} \left\{ \max \{ \bw^{\trsp}[\bx; 1], 0\}: \bw {\in} \mmR^{d+1}, \| \bw \| {=} 1 \right\}$,
		then with probability at least $ 1-\delta $,
		\begin{align}
			& \left| \mmE_{\mH} \left[d_{\mF, \phi} \left(\widehat{\mD}_{G_{\widehat{w}}}, {\mD}_{\emph{\real}} \right) \right]			{-} \inf_{ \mG } \mmE_{\mH} \left[ d_{\mF, \phi} \left({\mD}_{G_{u}},  {\mD}_{\emph{\real}} \right) \right] \right| \nonumber\\
			\leq& 2 \Delta \sqrt{\frac{2}{N} \log \left(\frac{1}{\delta}\right)} + \frac{4\sqrt{2}}{\sqrt{N}} + 2\epsilon(d_{\mM}).
		\end{align}
	\end{coll}
	
	In Corollary \ref{coll:generalization_relu_D}, using a one-layered ReLU network, the generalization bound of the proposed method is related to the error term \wrt the dimension of the latent distribution. 
	In other words, with a low dimensional input and sufficient training data, LCCGAN++ is able to obtain better generator approximation, and thus achieves better generalization performance in practice.

	\section{Experiments} \label{sec:experiment}
	We compare our method with several baseline methods, including DCGAN~\cite{radford2015unsupervised}, VAE~\cite{kingma2013auto}, 	WGAN-GP~\cite{gulrajani2017improved}, AGE~\cite{ulyanov2017adversarial}, StackGAN~\cite{zhang2017stackgan}, Progressive GAN~\cite{karras2017progressive} and LCCGAN~\cite{cao2018adversarial}. 
	% 	Besides, we consider some real-world datasets, 
	We conduct experiments on several benchmark datasets,
	including MNIST \cite{lecun1998gradient}, Oxford-102 \cite{nilsback2008automated}, LSUN \cite{song2015construction}, CelebA \cite{liu2015deep} and ImageNet \cite{deng2009imagenet}.
	% 	Considering the reproducibility, 
	We have made the code for both LCCGAN\footnote{\href{https://github.com/guoyongcs/LCCGAN}{https://github.com/guoyongcs/LCCGAN}.} and LCCGAN++\footnote{\href{https://github.com/guoyongcs/LCCGAN-v2}{https://github.com/guoyongcs/LCCGAN-v2}.} available on the internet.
	
	For the quantitative evaluation, we use some widely used metrics, \ie Inception Score (IS)~\cite{salimans2016improved} and Fr\'echet Inception Distance (FID)~\cite{heusel2017gans} and intra-FID \cite{miyato2018cgans}, to evaluate the generated samples.
	Specifically, IS measures both the single image quality and the diversity over a large number of samples (\textit{i.e.}, 50k), and a larger IS value corresponds to the better performance of the method. 
	FID and intra-FID measure the similarity between real and generated images, and a smaller value indicates the better performance. Note that these metrics are highly consistent with human evaluations.
	
	\begin{table*}[htp]
		\centering
		{
			\caption{
				Visual comparisons of GANs with different dimensions of the latent distribution on Oxford-102.
			}
			\label{fig:flower-result}
			\includegraphics[width=0.85\linewidth]{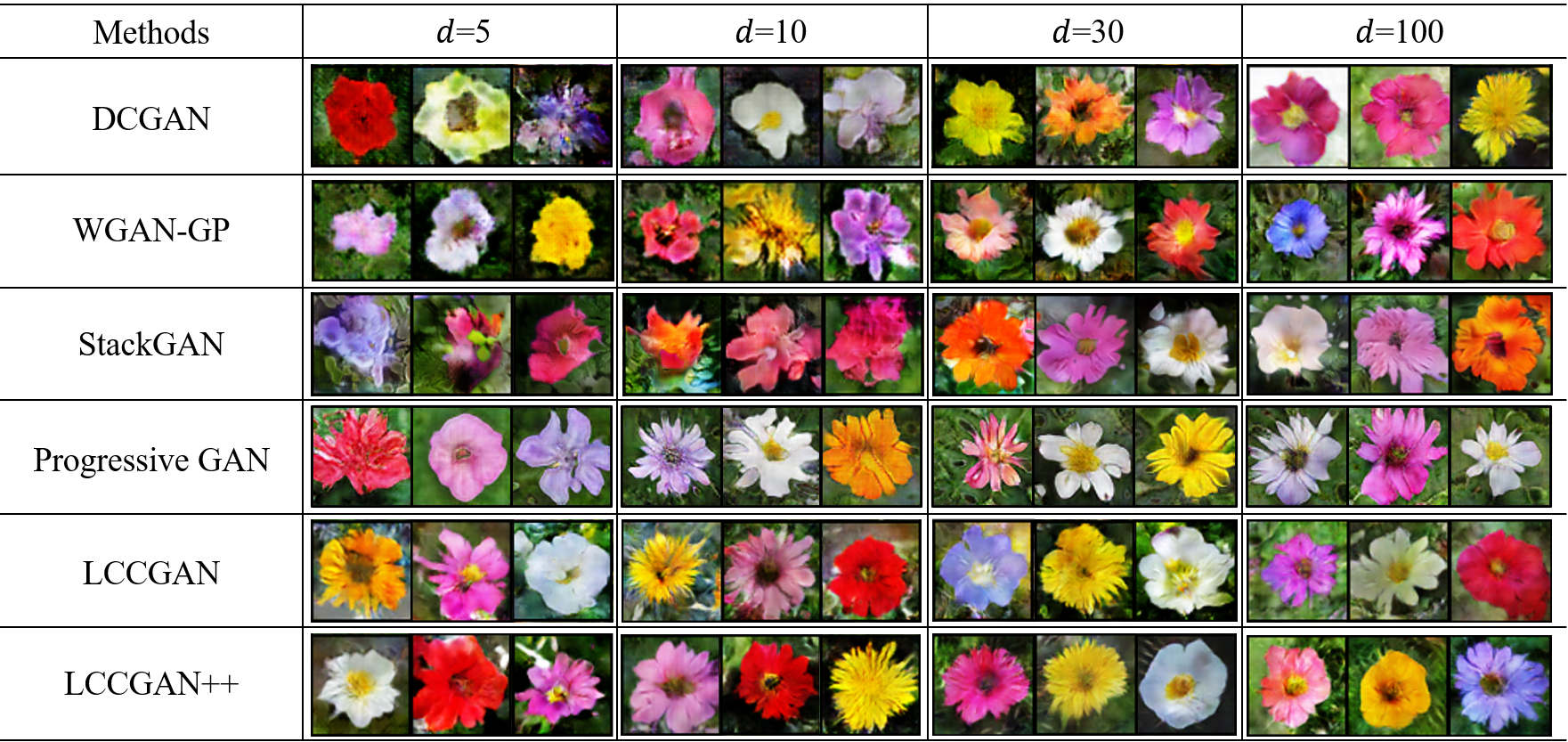}
		}
	\end{table*}
	
	\begin{table*}[htp]
		\normalsize
		\centering
		\caption{Comparisons of different GANs in terms of IS and FID on Oxford-102.}
		\resizebox{0.85\textwidth}{!}{
			\renewcommand{\arraystretch}{1.15}
			\begin{tabular}{c|c|c|c|c|c|c|c|c}	
				\hline %\toprule
				\multicolumn{1}{c|}{\multirow{2}[0]{*}{Methods}} & \multicolumn{2}{c|}{$ d=5 $} & 
				\multicolumn{2}{c|}{$ d=10 $} & 
				\multicolumn{2}{c|}{$ d=30 $} & 
				\multicolumn{2}{c}{$ d=100 $}\\
				\cline{2-9}
				& \multicolumn{1}{c|}{IS} & 
				\multicolumn{1}{c|}{FID}
				& \multicolumn{1}{c|}{IS} & 
				\multicolumn{1}{c|}{FID}
				& \multicolumn{1}{c|}{IS} & 
				\multicolumn{1}{c|}{FID}
				& \multicolumn{1}{c|}{IS} & 
				\multicolumn{1}{c}{FID} \\
				\hline %\midrule
				DCGAN~\cite{radford2015unsupervised}
				& 2.355 $\pm$ 0.019& 187.5
				& 3.262 $\pm$ 0.022& 204.7
				& 3.050 $\pm$ 0.015& 186.2 
				& 2.683 $\pm$ 0.022& 182.2 \\
				VAE~\cite{kingma2013auto}
				& 2.451 $\pm$ 0.018& 245.6
				& 2.358 $\pm$ 0.022& 190.6
				& 2.234 $\pm$ 0.016& 244.0 
				& 2.856 $\pm$ 0.024& 214.8 \\
				WGAN-GP~\cite{gulrajani2017improved}
				& 2.719 $\pm$ 0.031& 185.2
				& 2.891 $\pm$ 0.025& 179.8 
				& 3.081 $\pm$ 0.018& 136.7
				& 3.458 $\pm$ 0.028& 160.4 \\
				AGE~\cite{ulyanov2017adversarial}
				& 2.865 $\pm$ 0.024& 234.1
				& 3.062 $\pm$ 0.021& 186.7
				& 2.630 $\pm$ 0.023& 211.8
				& 2.488 $\pm$ 0.014& 235.9 \\
				StackGAN~\cite{zhang2017stackgan} 
				& 2.664 $\pm$ 0.013& 164.2
				& 2.702 $\pm$ 0.015& 167.7 
				& 3.109 $\pm$ 0.018& 197.0  
				& 2.741 $\pm$ 0.022& 178.8 \\
				Progressive GAN~\cite{karras2017progressive}
				& 2.844 $\pm$ 0.031& 128.6 
				& 3.295 $\pm$ 0.028& 128.6
				& 3.196 $\pm$ 0.028& 106.8 
				& 3.532 $\pm$ 0.028& 114.5 \\
				\hline %\midrule
				LCCGAN~\cite{cao2018adversarial}
				& 3.079 $\pm$ 0.026& 71.2
				& 3.077 $\pm$ 0.033& 82.7
				& 3.003 $\pm$ 0.030& 61.9 
				& 3.147 $\pm$ 0.038& 66.7 \\
				LCCGAN++
				& \bf3.267 $\pm$ 0.023& \bf71.0 
				& \bf3.394 $\pm$ 0.019& \bf71.1 
				& \bf3.370 $\pm$ 0.031& \bf57.7 
				& \bf3.590 $\pm$ 0.020& \bf60.7 \\
				\hline %\bottomrule
			\end{tabular}
		}
		\label{tab:flowers}
	\end{table*}

	\subsection{Comparisons on MNIST}
	\label{subsubsection:mnist}
	In this experiment, we compare different GANs on MNIST. 
	From Fig.~\ref{fig:mnist_low},
	when $d{=}3$, DCGAN and StackGAN produce only few kinds of digits with almost the same shapes. 
	In contrast, LCCGAN often produces digits with different styles and orientations.
	Furthermore, LCCGAN++ further produces images with better visual fidelity and higher diversity.
	Equipped with LCC, the proposed method effectively preserves the local information of data and thus helps the training of GANs.
	
	From Fig.~\ref{fig:mnist_high}, when we increase the dimension of input to $d{=}5$, the considered baseline methods often produce the digits with distorted structures.
	In contrast, with such a low dimensional input, LCCGAN is able to produce the images with meaningful content. Furthermore, LCCGAN++ significantly outperforms the considered baseline methods and produces sharper images.
	More critically, with the help of LCC coding, LCCGAN and LCCGAN++ with $d{=}5$ are able to achieve comparable or even better performance than their GAN counterparts with $d{=}100$ (See the red box in Fig.~\ref{fig:mnist_high}).
	These results show the effectiveness of the proposed method in training generative models by exploiting the local information of the latent manifold.
	
	\subsection{Comparisons on Oxford-102}
	\label{subsubsection:flowers}
	We further evaluate our method on Oxford-102, and investigate the effect of different input dimensions.
	The qualitative and quantitative results are shown in Table~\ref{fig:flower-result} and Table~\ref{tab:flowers}, respectively.
	
	\textbf{Qualitative results.}
	From Table~\ref{fig:flower-result}, we have the following observations. 
	First, the performance of the baselines highly depends on the input dimension. 
	For example, given a low dimension with $d{=}5$ or $d{=}10$, DCGAN often generates images with a blurred structure and distorted regions. 
	In contrast, 
	our method is able to produce realistic images.
	Second, we further investigate the effect of the input dimension on the quality of the generated images. 
	When $d{=}100$, LCCGAN++ consistently outperforms LCCGAN and the considered baselines.
	
	\textbf{Quantitative results.}
	From Table \ref{tab:flowers}, when $d{=}5$, Progressive GAN obtains slightly better IS and FID than other methods. 
	In contrast, LCCGAN and LCCGAN++ significantly outperform the other methods with various $d$ in terms of both IS and FID.  
	More critically, LCCGAN++ with $d{=}5$ achieves even better performance than all baselines with $d{=}30$ and several methods with $d{=}100$, \textit{e.g.,} DCGAN.
	It means that our method only requires a low-dimensional input to achieve good performance.
	These results show the effectiveness of our method. 
	% 	in producing perceptually promising images with high quality and large diversity. 

	\subsection{Comparisons on CelebA} \label{subsubsection:celeba}
	We also conduct experiments on the CelebA dataset \cite{liu2015deep}.  
	Due to the difficulty of producing face images, we use a larger input dimension (\eg $d{=}30$) to train the generative models. 
	% 	in this experiment. 
	% 	The qualitative and quantitative results are shown in Table~\ref{fig:celeba-result} and Table~\ref{tab:celeba}, respectively. 
	
	\textbf{Qualitative results.} 
	In Table~\ref{fig:celeba-result}, by introducing LCC sampling into the training, our method with a low input dimension $d{=}30$ produce promising face images with better quality and larger diversity than DCGAN and Progressive GAN with $d{=}100$.  
	Moreover, given the same input dimension, our proposed LCCGAN++ shows better performance than LCCGAN and other baseline methods. 
	More qualitative results are put in Supplementary materials.
	
	\textbf{Quantitative results.}
	In Table~\ref{tab:celeba}, our LCCGAN yields comparable results with state-of-the-art GANs. With the improved LCC, LCCGAN++ further improves the performance and outperforms the other methods with various $d$. 
	% 	In particular, when the input dimension $d{=}30$, LCCGAN++ achieves the best quantitative results with an FID score of 29.2. 
	These results imply that our method is able to generate face images with high quality and large diversity even when the input dimension is low.
	% 	Our LCCGAN++ greatly benefits from the LCC sampling method and the improved generator approximation, and thus makes the training much easier than directly using the global coordinate codings.

	\subsection{Comparisons on LSUN} \label{subsubsection:lsun}
	We conduct experiments on LSUN \cite{song2015construction} to evaluate the performance of our proposed method.
	% 	The qualitative and quantitative results are shown in Table \ref{fig:lsun-result} and Table \ref{tab:lsun}, respectively.

	\textbf{Qualitative results.}
	In Table~\ref{fig:lsun-result}, given a low dimension of the input (\ie $d{=}10$), LCCGAN and LCCGAN++ are able to produce images with sharper structures and richer details, and thus consistently outperform the considered baselines. 
	In contrast, WGAN-GP and Progressive GAN fail to produce meaningful bedroom images.
	More importantly, the quality of generated images by LCCGAN and LCCGAN++ with $d{=}10$ are even better than that of WGAN-GP and Progressive GAN with $d{=}100$. %, which often require a large dimensional input.  
	
	\textbf{Quantitative results.}
	In Table~\ref{tab:lsun}, the performance of our  method is generally better than the considered baseline methods in terms of the lowest FID score and comparable IS value.
	It implies that our method is able to generate images with high quality and large diversity.
	Although Progressive GAN achieves a good IS with $d{=}10$ on LSUN-bedroom, LCCGAN++ achieves the lower FID score and outperforms Progressive GAN. 
	% 	by approximately 46.9.
	% 	Moreover, with the help of the improved generator approximation, LCCGAN++ outperforms LCCGAN in terms of FID.
	
	\begin{table*}[htp]
		\centering
		{
			\caption{		
				Visual comparisons of different GANs with different input dimensions on the LSUN-bedroom and LSUN-classroom datasets.
			}
			\label{fig:lsun-result}
			\includegraphics[width=0.9\linewidth]{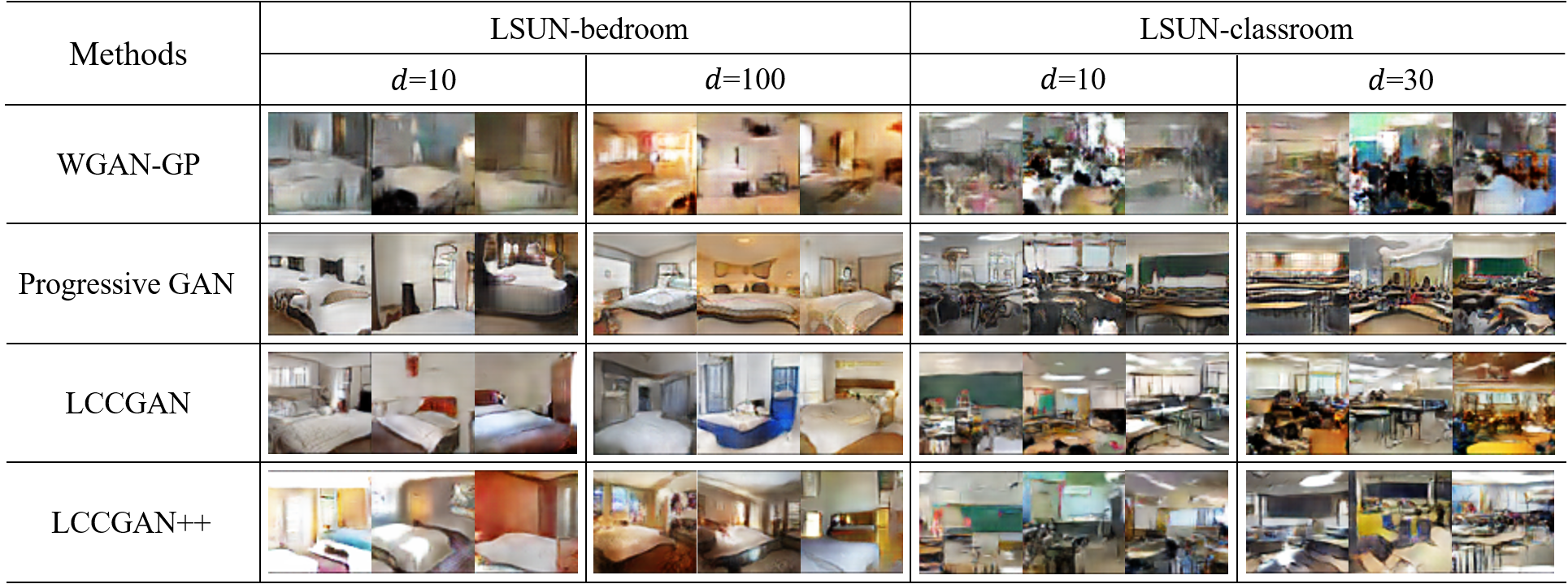}
		}
	\end{table*}
	
	\begin{table*}[htp]
		\normalsize
		\centering
		\caption{Comparisons with different GANs with different dimensions of the latent distribution in terms of IS and FID on LSUN.}
		\resizebox{0.9\textwidth}{!}{
			\renewcommand{\arraystretch}{1.25}
			
			\begin{tabular}{c|c|c|c|c|c|c|c|c|c|c|c|c|c|c|c|c}	
				\hline %\toprule
				\multicolumn{1}{c|}{\multirow{3}[0]{*}{Methods}}&
				\multicolumn{8}{c|}{LSUN-bedroom} &
				\multicolumn{8}{c}{LSUN-classroom} \\
				\cline{2-17} 
				& \multicolumn{2}{c|}{$ d=5 $} 
				& \multicolumn{2}{c|}{$ d=10 $} 
				& \multicolumn{2}{c|}{$ d=30 $} 
				& \multicolumn{2}{c|}{$ d=100 $}
				& \multicolumn{2}{c|}{$ d=5 $} 
				& \multicolumn{2}{c|}{$ d=10 $} 
				& \multicolumn{2}{c|}{$ d=30 $} 
				& \multicolumn{2}{c}{$ d=100 $}\\ 
				\cline{2-17}
				& \multicolumn{1}{c|}{IS} 
				& \multicolumn{1}{c|}{FID}
				& \multicolumn{1}{c|}{IS} 
				& \multicolumn{1}{c|}{FID}
				& \multicolumn{1}{c|}{IS} 
				& \multicolumn{1}{c|}{FID}
				& \multicolumn{1}{c|}{IS} 
				& \multicolumn{1}{c|}{FID}
				& \multicolumn{1}{c|}{IS} 
				& \multicolumn{1}{c|}{FID}
				& \multicolumn{1}{c|}{IS} 
				& \multicolumn{1}{c|}{FID}
				& \multicolumn{1}{c|}{IS} 
				& \multicolumn{1}{c|}{FID}
				& \multicolumn{1}{c|}{IS} 
				& \multicolumn{1}{c}{FID} \\
				\hline %\midrule
				DCGAN ~\cite{radford2015unsupervised}
				& 1.969 & 253.7
				& 2.531 & 193.9
				& 2.409 & 204.6  
				& 2.165 & 239.7 
				
				& 2.230 & 272.2
				& 2.204 & 258.8
				& 2.401 & 233.1
				& 2.347 & 271.9 
				\\
				VAE ~\cite{kingma2013auto}
				& 2.785 & 198.7  
				& 2.967 & 183.3
				& 3.218 & 166.3 
				& 3.265 & 178.9 
				
				& 2.195 & 232.7
				& 2.491 & 164.0
				& 2.646 & 182.4
				& 2.740 & 175.4
				\\
				WGAN-GP ~\cite{gulrajani2017improved}
				& 2.875 & 172.4
				& 2.834 & 176.3 
				& 2.950 & 154.2
				& 2.965 & 172.6 
				
				& 2.595 & 195.7
				& 2.733 & 197.6
				& 2.799 & 169.7
				& 2.701 & 173.3
				\\
				AGE ~\cite{ulyanov2017adversarial}
				& 2.031 & 312.1
				& 2.345 & 193.8
				& 2.186 & 219.3
				& 2.602 & 171.6 
				
				& 2.002 & 311.0
				& 2.142 & 267.3
				& 2.278 & 262.7
				& 1.956 & 321.5
				\\
				StackGAN ~\cite{zhang2017stackgan} 
				& 2.722 & 237.3
				& 2.637 & 197.3 
				& 2.675 & 164.5  
				& 2.612 & 238.0 
				
				& 2.292 & 209.7
				& 1.961 & 239.0
				& 2.340 & 256.2
				& 1.855 & 257.0
				\\
				Progressive GAN ~\cite{karras2017progressive}
				& 3.405 & 161.4 
				& \bf3.763 & 156.7
				& \bf3.951 & 149.3
				& 3.837 & 154.3 
				
				& 2.673 & 189.2
				& 3.073 & 174.9
				& \bf3.367 & 170.9
				& 3.176 & 177.8
				\\
				\hline %\midrule
				LCCGAN ~\cite{cao2018adversarial}
				& 3.254 & 104.1
				& 3.213 & 110.3
				& 3.084 & 139.1 
				& 3.350 & 115.0 
				
				& 2.786 & 105.3
				& \bf3.094 & 103.0
				& 2.974 & 103.4
				& 2.532 & 132.2
				\\
				LCCGAN++
				& \bf3.406 & \bf98.0 
				& 3.683 & \bf109.8 
				& 3.546 & \bf88.1 
				& \bf4.109 & \bf110.7 
				
				& \bf2.866 & \bf95.2
				& 3.005 & \bf96.6
				& 3.201 & \bf102.9
				& \bf3.273 & \bf98.9
				\\
				\hline %\bottomrule
				
			\end{tabular}
		}
		\label{tab:lsun}
	\end{table*}

	\begin{table}[h]
		\centering
		{
			\caption{		
				Comparisons of GANs with different dimensions on CelebA.
			}
			\label{fig:celeba-result}
			\includegraphics[width=0.8\linewidth]{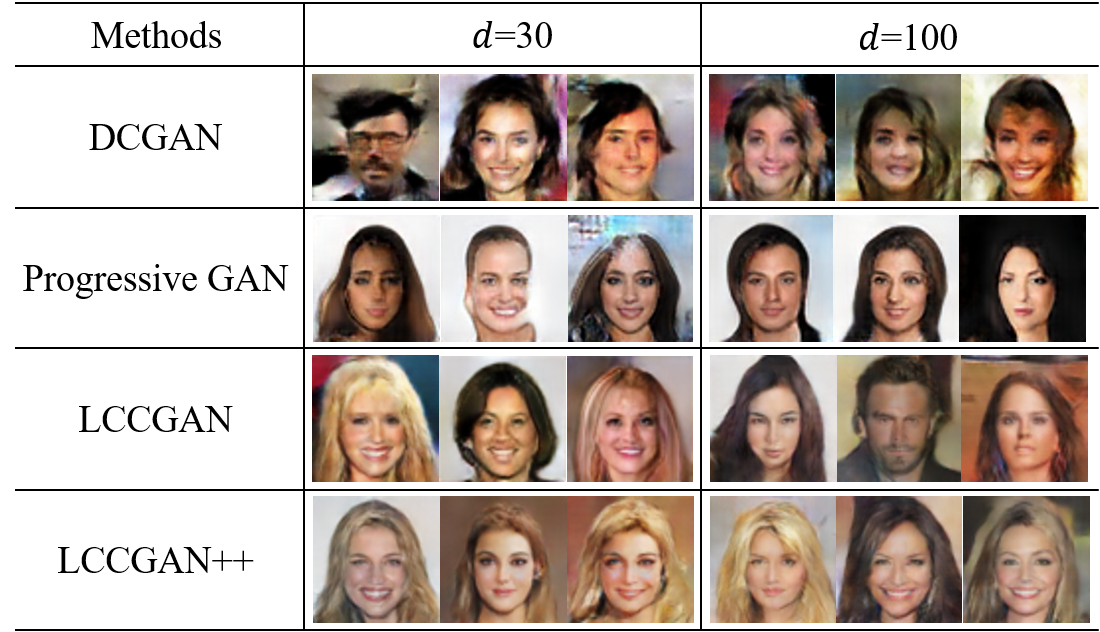}
		}
	\end{table}
	
	\begin{table}[h]
		\normalsize
		\centering
		\caption{Comparisons of GANs in terms of IS and FID on CelebA.}
		\resizebox{0.5\textwidth}{!}{
			\renewcommand{\arraystretch}{1.15}
			
			\begin{tabular}{c|c|c|c|c}	
				\hline %\toprule
				\multicolumn{1}{c|}{\multirow{2}[0]{*}{Methods}} & 
				\multicolumn{2}{c|}{$ d=30 $} & 
				\multicolumn{2}{c}{$ d=100 $}\\
				\cline{2-5}
				& \multicolumn{1}{c|}{IS} & 
				\multicolumn{1}{c|}{FID}
				& \multicolumn{1}{c|}{IS} & 
				\multicolumn{1}{c}{FID} \\
				\hline %\midrule
				DCGAN ~\cite{radford2015unsupervised}
				& 2.299 $\pm$ 0.014& 67.2 
				& 2.214 $\pm$ 0.022& 78.5 \\
				VAE ~\cite{kingma2013auto}
				& 2.395 $\pm$ 0.017& 52.0 
				& 2.308 $\pm$ 0.019& 54.4 \\
				%			W ~\cite{arjovsky2017wasserstein}
				%			& 2.212 $\pm$ 0.015& 0.371
				%			& 2.239 $\pm$ 0.028& 0.426 \\
				WGAN-GP ~\cite{gulrajani2017improved}
				& 2.344 $\pm$ 0.025& 92.0
				& 2.388 $\pm$ 0.023& 88.9 \\
				AGE ~\cite{ulyanov2017adversarial}
				& 2.517 $\pm$ 0.025& 82.2
				& 2.612 $\pm$ 0.026& 63.0 \\
				StackGAN ~\cite{zhang2017stackgan} 
				& 2.036 $\pm$ 0.016& 131.0  
				& 2.419 $\pm$ 0.014& 133.8 \\
				Progressive GAN ~\cite{karras2017progressive}
				& 2.527 $\pm$ 0.020& 52.8
				& 2.530 $\pm$ 0.017& 55.2 \\
				\hline %\midrule
				LCCGAN ~\cite{cao2018adversarial}
				& 2.420 $\pm$ 0.027& 54.4 
				& 2.526 $\pm$ 0.025& 31.9 \\
				LCCGAN++ 
				& \bf2.582 $\pm$ 0.018& \bf29.2 
				& \bf2.625 $\pm$ 0.017& \bf25.9 \\
				\hline %\bottomrule
			\end{tabular}
		}
		\label{tab:celeba}
	\end{table}	
	
	\begin{table*}[ht]
		\centering
		\caption{Effect of the LCC training method on improving the performance of different GANs on Oxford-102.}
		\resizebox{1\textwidth}{!}{
			\begin{tabular}{c|cc|cc|cc|cc|cc}
				\hline
				\multirow{2}{*}{Method} & \multicolumn{2}{c|}{DCGAN} & \multicolumn{2}{c|}{WGAN-GP} & \multicolumn{2}{c|}{StackGAN-v1} & \multicolumn{2}{c|}{StackGAN-v2} & \multicolumn{2}{c}{Progressive GAN} \\
				\cline{2-11}
				& IS & FID & IS & FID & IS & FID & IS & FID & IS & FID\\
				\hline
				Baseline 
				& 2.683 $\pm$ 0.022 & 182.2
				& 3.458 $\pm$ 0.028 & 160.4
				& 2.741 $\pm$ 0.022 & 178.8
				& 3.087 $\pm$ 0.027 & 27.0      
				& 3.532 $\pm$ 0.028 & 114.5
				\\
				with LCC ($q{=}2$) 
				& 3.003 $\pm$ 0.030 & 61.9    
				& 3.496 $\pm$ 0.032 & 155.5
				& 2.895 $\pm$ 0.017 & 177.6    
				& 3.088 $\pm$ 0.031 & 23.7       
				& 3.571 $\pm$ 0.024 & 111.2
				\\
				with LCC ($q{=}3$)
				& \bf3.370 $\pm$ 0.031 & \bf57.7
				& \bf3.546 $\pm$ 0.032 & \bf145.9
				& \bf3.005 $\pm$ 0.014 & \bf168.2       
				& \bf3.216 $\pm$ 0.030 & \bf22.2   
				& \bf3.710 $\pm$ 0.036 & \bf109.6\\
				\hline
			\end{tabular}%
		}	
		\label{tab:comp_stackgan_v2}
	\end{table*}
	
	\begin{table}[t]
		\caption{Visual comparisons of different GANs on ImageNet, including Promontory and Volcano. Here, we use LCCGAN++ as our method.}
		\label{tab:imagenet}
		\centering
		\includegraphics[width=0.95\linewidth]{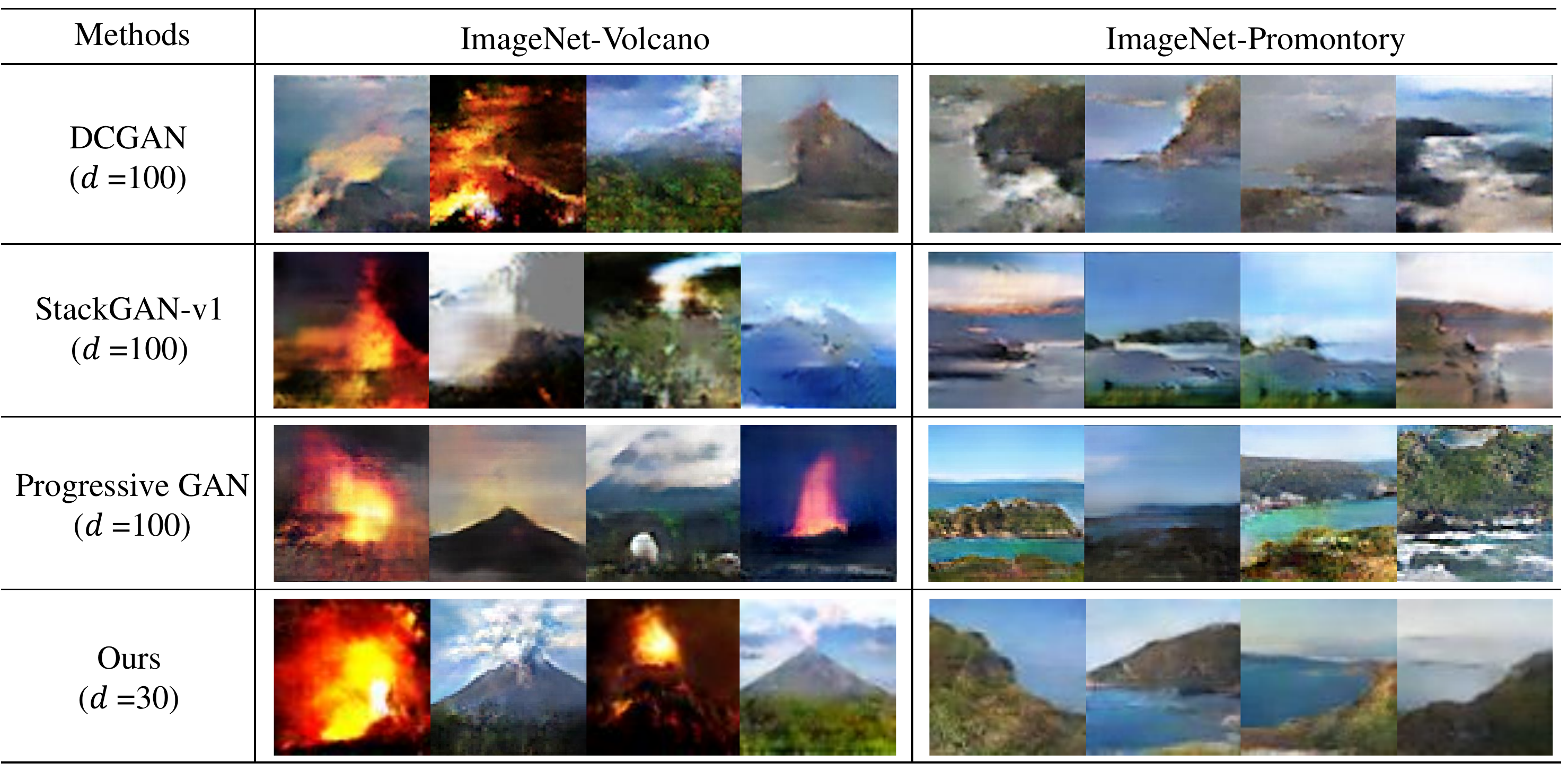} %sampling-result.png
	\end{table}

	\subsection{Comparisons on ImageNet}
	In this experiment, we further evaluate the performance of the proposed LCCGAN++ on the ImageNet dataset.
	Specifically, since we focus on unconditional GAN models in this paper, training 1000 models on the ImageNet dataset (1000 categories in total) is infeasible and impractical. Following the previous studies~\cite{guo2019auto, zhang2018stackgan++}, we conduct experiments on two categories of the ImageNet dataset, \ie Promontory and Volcano.
	% 	to compare the performance of different GAN models. 
	% 	including DCGAN, StackGAN and Progressive GAN.
	% 	We show the visual comparisons of these GAN models in Table \ref{tab:imagenet}.
	
	From Table \ref{tab:imagenet}, with a low-dimensional input $d{=}30$, our proposed LCCGAN++ is able to produce promising images for both Promontory and Volcano.
	More importantly, the proposed LCCGAN++ with $d{=}30$ has better quality than the considered baseline methods with a high dimension of $d{=}100$ on these two categories.
	Therefore, these results demonstrate the effectiveness of our proposed method with a low dimension of the input.
	Moreover, our method has good generalization performance even when the input dimension is low.

	\subsection{Effectiveness of the LCCGAN Framework}
	In this experiment, we verify the effectiveness of the LCCGAN framework by introducing LCC into different GANs, including DCGAN, WGAN-GP, StackGAN-v1, StackGAN-v2 and Progressive GAN.
	% 	Note that we build our LCCGAN based on the DCGAN model (with 3.6M parameters) which is much smaller than StackGAN-v2 (with 16.5M parameters) and Progressive GAN (with 60.7M parameters). 
	% 	In this sense, due to the non-negligible gap in model size, it seems unfair to directly compare the LCC based DCGAN with larger GAN models, like StackGAN-v2 and Progressive GAN. 
	Since we build our LCCGAN based on the DCGAN model (with 3.6M parameters), it seems unfair to directly compare the LCC based DCGAN with larger GAN models, like StackGAN-v2 (with 16.5M parameters) and Progressive GAN (with 60.7M parameters). 
	% 	To address this issue, we verify the effectiveness of the proposed LCC method by applying our method to other GANs. In this way, we can compare the performance of the models equipped with and without LCC sampling. 
	% 	We show the results in Table~\ref{tab:comp_stackgan_v2}. 
	From Table~\ref{tab:comp_stackgan_v2}, the resultant models with LCC %trained with both LCC-v1 and LCC-v2 
	consistently outperform the baseline models given different dimensions of the input, which demonstrates the effectiveness of our method.

	\section{Additional Experiments} \label{exp:discussions}
	
	\subsection{Demonstration of LCC Sampling}\label{subsection:sensitive_lcc_sampling}
	In this experiment,
	we investigate the effectiveness of the LCC sampling. 
	% In this way, 
	% we are able to verify that the proposed method does not memorize the training data by finding the closest image from the training set. 
	Specifically, we first randomly select one latent point in the coordinate system and find the nearest $d$ bases.
	Then, we generate 10 latent points using random weights based on the selected $d$ bases to produce images.
	% 	Here, we limit the nearest neighbor bases since we use the local coordinate system to generate samples over the latent manifold.
	% 	With the nearest neighbor bases, we are able to sample meaningful points to exploit the local information of real data and improve the performance of GANs.
	% 	Ideally, these images should be located in a local area of the latent manifold and share common features.
	% 	To verify this, 
	% 	We conduct experiments on several benchmark datasets. 
	From  Table~\ref{fig:LCCGANsampling}, the proposed method is able to produce images with different orientations or styles.
	With the help of LCC sampling, our model generalizes well to unseen data rather than simply memorizing the training samples.
	These results demonstrate the effectiveness of the proposed sampling method in exploiting the local information of data.

	\subsection{Latent Manifold Interpolations}
	
	To further verify the generalization performance of our method, we conduct latent manifold interpolations on the Oxford-102 dataset. 
	% 	First, we apply our LCC sampling method to generate two images in the same local coordinate system.
	% 	Given these two generated images (See the first and the last column of Table \ref{fig:interpolation}), we have two corresponding LCC codings, and then linearly interpolate several codings between these two LCC codings.
	Specifically, we first apply our LCC sampling method to generate two images in the same local coordinate system, and we have two corresponding LCC codings. 
	Then, we linearly interpolate a set of codings between these two LCC codings of two given images.
	From Table \ref{fig:interpolation}, our proposed method is able to interpolate realistic and smooth generated images.
	% 	Moreover, the quality of the interpolated images is high, indicating these images may be guaranteed on the latent manifold.
	% 	With the help of our LCC sampling method, we not only exploit the local information of real data, but also explore the smooth properties of the generator in the local coordinate system.  
	% 	These results imply that our method is able to generate new data without memorizing the training data.
	These results imply that our method is able to explore the smooth properties of the generator in the local coordinate system.

	\subsection{Comparisons of High-resolution Image Generation}
	We compare the performance of different GAN models equipped with and without LCC sampling when producing high-resolution images.
	% 	Specifically, since the proposed LCCGAN is a general GAN framework, and it has good scalability to generate high-resolution images.
	In this experiment, we apply the LCC learning method to several GAN models, such as DCGAN, StackGAN-v2, and Progressive GAN. 
	From Table~\ref{fig:high_resolution}, with a low input dimension $d{=}30$, the models with the LCC are able to generate more photo-realistic high-resolution images than the baseline models with $d{=}100$ under the resolutions of $ 128{\times}128 $ and $ 256{\times}256 $. 
	It implies that our proposed method is able to generate high-resolution images even when the input dimension is low.

	\begin{table}[t]
		\caption{Generated images from LCC sampling on MNIST, Oxford-102 and CelebA. The second column shows the images generated from the synthesized latent points. In the last column, we use the Pearson distance to find the closest image in the training data. 
		}
		\label{fig:LCCGANsampling}
		\centering
		\includegraphics[width=0.95\linewidth]{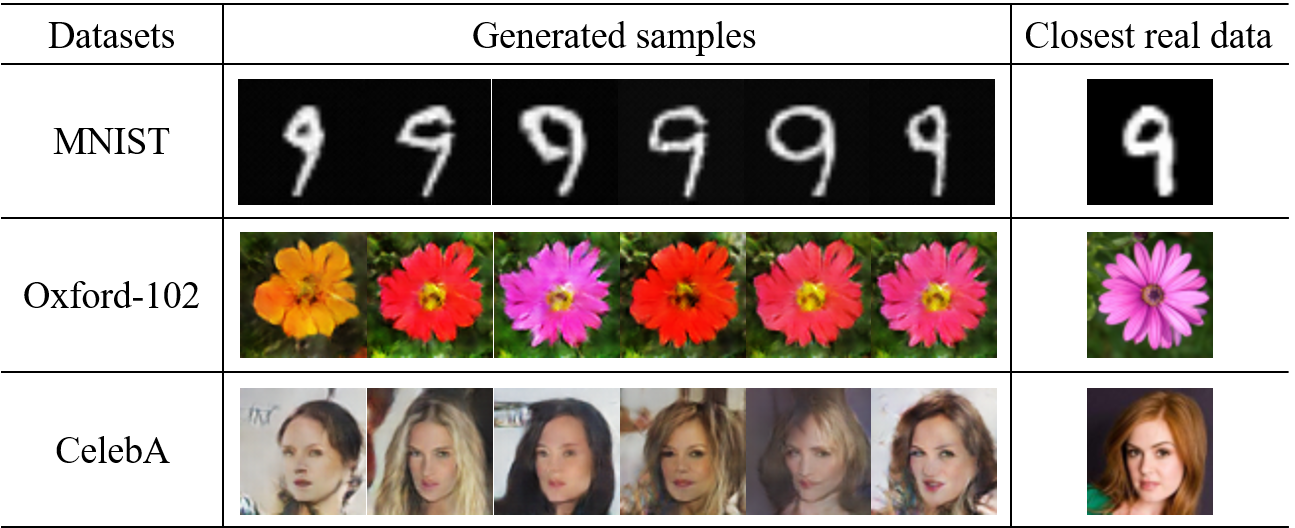} %sampling-result.png
	\end{table}
	
	\begin{table}[t]
		\caption{Interpolations between two generated images on Oxford-102. 
			The first and the last column show the generated images, and the middle column is the interpolated images between two corresponding images.}\label{fig:interpolation}
		\centering
		\includegraphics[width=1\linewidth]{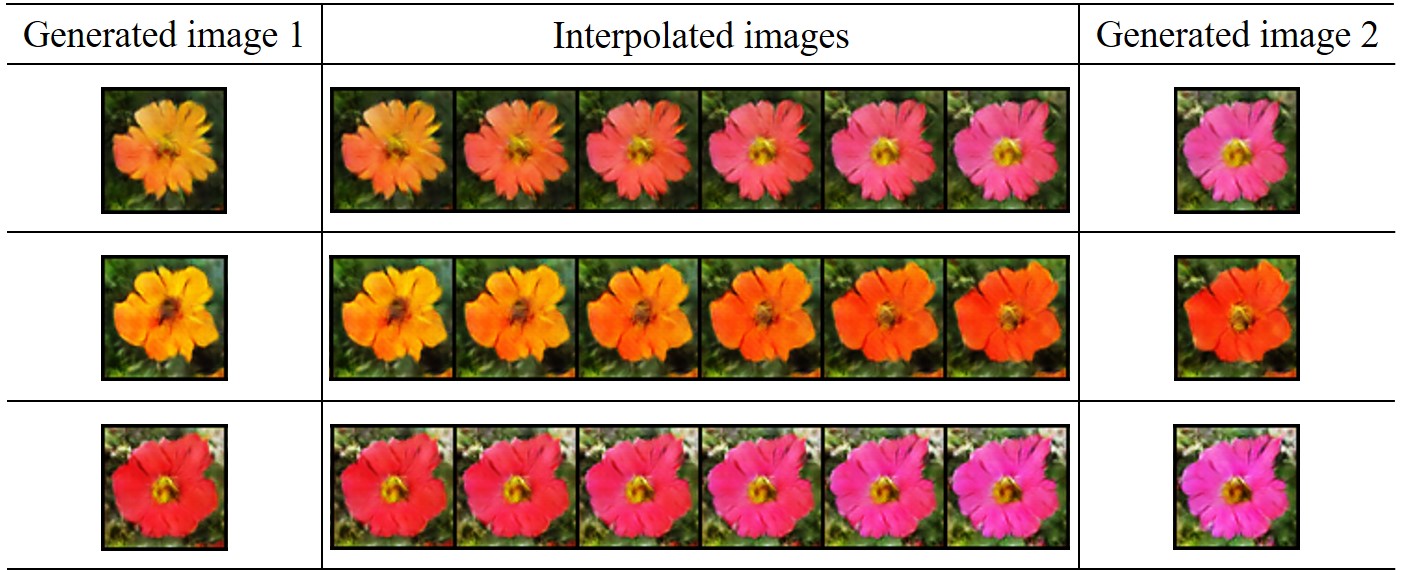} 
	\end{table}
	
	\begin{table*}[h]
		\centering
		\caption{
			Comparisons of different GAN models equipped with and without the LCC sampling method. Here, we train our method using LCC $(q{=}3)$. 
		}
		\label{fig:high_resolution}
		\includegraphics[width=0.9\linewidth]{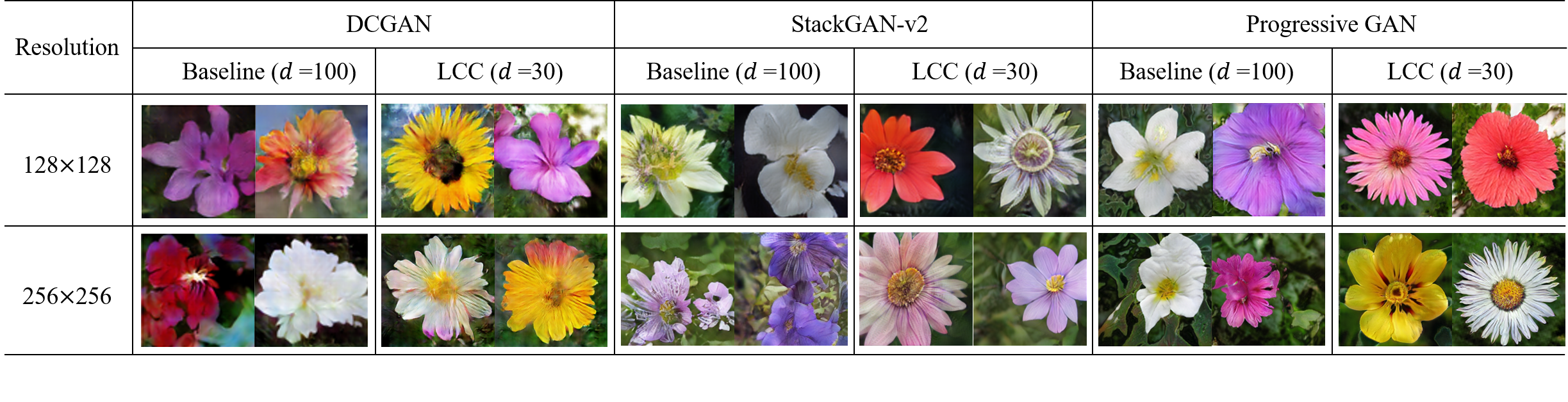} 
	\end{table*}
	
	\begin{table}[t]
		\caption{
			Effect of $ d_B $ and $ M $ on the performance of LCCGAN++ on Oxford-102.
		}\label{fig:ablation_dB_M}
		\centering
		\includegraphics[width=1\linewidth]{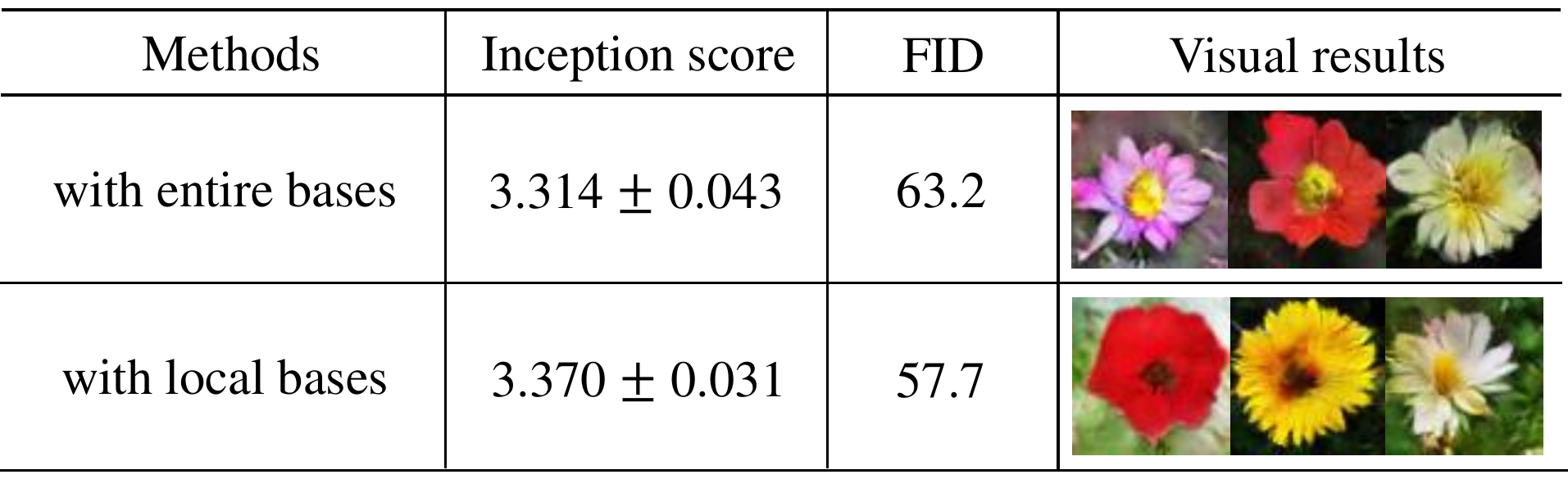} 
	\end{table}

	\subsection{Comparisons between Local and Entire Bases}
	In this experiment, we compare the LCCGAN model with local bases and the model with entire bases.
	% 	This experiment verifies the effectiveness of  using local bases.
	% 	Specifically, we implement our method by directly sampling random $ \bgamma $ and computing $ \bV\bgamma $.
	%{To use the entire bases, we randomly initialize $ \bgamma $ for all bases.}
	From Table \ref{fig:ablation_dB_M}, LCCGAN with local bases has the largest IS and the lowest FID, and thus generates the most realistic images (as shown in the last column).
	It means that LCCGAN using local bases is able to exploit local information to improve the quality of generated images.
	In contrast, using the entire bases would sample meaningless points to generate images with poor quality.
	These results demonstrate the effectiveness of our method using the local bases.

	\begin{table}[t]
		\centering
		\caption{Comparisons with different GANs in terms of intra-FID on LSUN. We set $d{=}30$ for all the experiments.}
		\label{table:intraFID}
		\resizebox{0.5\textwidth}{!}{
			\begin{tabular}{cccc}
				\hline
				Methods         & LSUN-classroom (FID) & LSUN-bedroom (FID) & Intra-FID \\ \hline
				DCGAN           & 182.38 & 212.82 & 197.60 \\
				StackGAN-v2     & 162.05 & 134.22 & 148.14 \\
				Progressive GAN & 178.17 & 165.01 & 171.59 \\ \hline
				LCCGAN       & 107.19 & 94.89  & 101.04 \\
				LCCGAN++       & \bf97.87 & \bf90.58 & \bf94.23        \\ \hline
			\end{tabular}
		}
	\end{table}
	
	\begin{table}[t]
		\centering
		\caption{Discussion on $L_{\bh}$ and $L_{\nu}$ on Oxford-102 with $d{=}30$. 
		}
		\resizebox{0.5\textwidth}{!}{
			\begin{tabular}{c|cccccc}
				\hline
				Settings of $L_{\nu}$     & 0.0001 & 0.001 & 0.01  & 0.1   & 1     & 10 \\
				\hline
				IS   & \textbf{3.370} & 2.817 & 2.247 & 2.949 & 2.296 & 2.035 \\
				FID   & \textbf{57.7} & 205.9 & 255.1 & 280.8 & 262.9 & 269.8 \\
				\hline
				Settings of $L_{\bh}$ & 0.0001 & 0.001 & 0.01  & 0.1   & 1     & 10 \\
				\hline
				IS   & 1.890 & 2.228 & 1.984   & 3.159  & \textbf{3.370} & 3.239 \\
				FID   & 247.9 & 286.8 & 240   & 77.1  & \textbf{57.7} & 58.0 \\
				\hline
			\end{tabular}
		}
		\label{tab:hyperparameter}
	\end{table}

	\begin{table}[t!]
		\caption{Effect of end-to-end training for LCCGAN++ on Oxford-102.}\label{tab:lccgan_end2end}
		\centering
		\includegraphics[width=0.97\linewidth]{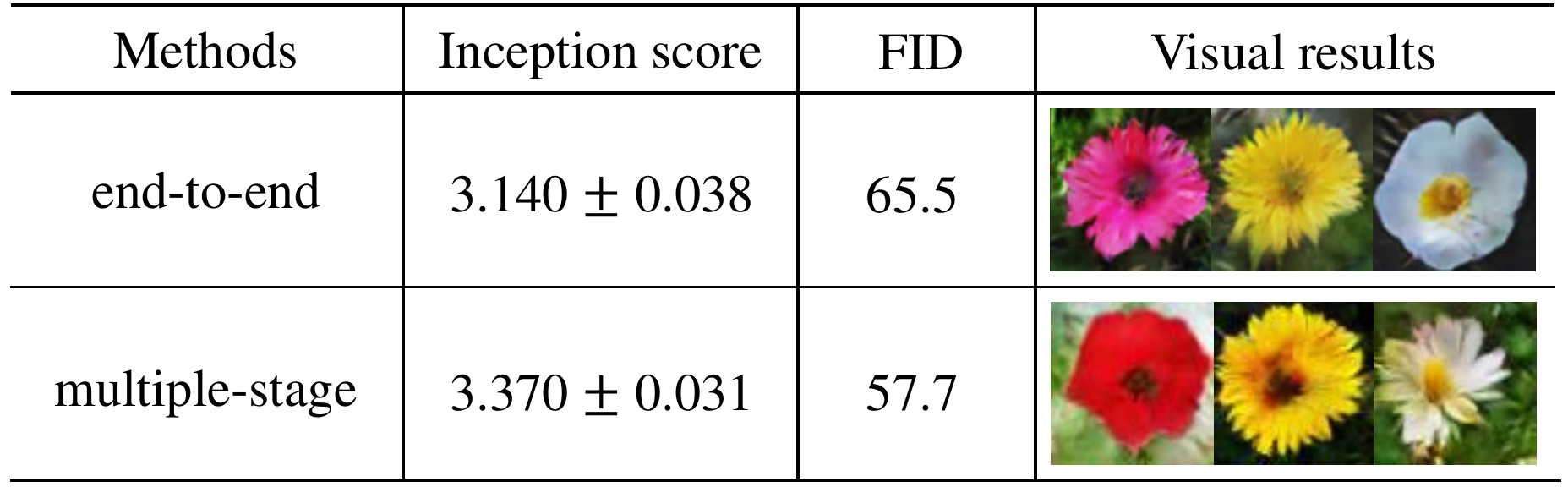} %sampling-result.png
	\end{table}

	\begin{table}[t!]
		\normalsize
		\centering
		\caption{{Effect} of $\#$class on LCCGAN++ in terms of IS and FID.}
		\label{table:ratio1}
		\resizebox{0.5\textwidth}{!}{
			\begin{tabular}{c|cc|cc|cc|cc}	
				\hline 
				\multirow{2}{*}{Input dimension} 
				& \multicolumn{2}{c|}{$\#$class${=}$30} 
				& \multicolumn{2}{c|}{$\#$class${=}$50} 
				& \multicolumn{2}{c|}{$\#$class${=}$70} 
				& \multicolumn{2}{c}{$\#$class${=}$102}\\
				\cline{2-9} 
				& IS & FID & IS & FID & IS & FID & IS & FID \\
				\hline 
				$ d=3 $ 
				& 3.087 & 104.8 
				& 2.902 & 119.9 
				& 2.742 & 125.8 
				& 2.697 & 130.1	\\
				$ d=5 $ 
				& 3.116 & 111.3  
				& 2.881 & 125.3 
				& 2.655 & 140.1 
				& 2.500 & 167.6	\\
				\hline 
			\end{tabular}
		}
	\end{table}

	\subsection{Comparisons in terms of Intra-FID}
	In this experiment, we train a single GAN model over different classes and evaluate the method using the intra-FID~\cite{miyato2018cgans}.
	Such a metric first computes an FID score separately for each condition/class and then reports the average score over all conditions.
	However, this paper focuses on unconditional GANs and they have no conditions/labels associated with the generated images.
	As a result, we cannot directly {compute} the intra-FID.
	To address this, we first train a classification model to classify the generated images into different classes, and then obtain the intra-FID score by computing an FID score for each class.
	
	We train the GAN models on two LSUN classes (\ie LSUN-classroom and LSUN-bedroom) and the classification model becomes a binary model (with the average accuracy of 95.1\%).
	We report both FID score for each class and the intra-FID scores of different methods in Table \ref{table:intraFID}.
	From these results, our LCCGAN yields the smallest intra-FID among all the considered methods.
	It means that LCCGAN and LCCGAN++ are able to generate diverse samples by capturing the local information of data for each class.
	Moreover, LCCGAN++ achieves better performance than LCCGAN with the same input dimension because LCCGAN++ has better approximation of generative models.
	% In contrast, other GANs have poor diversity especially when the input dimension is low. 
	
	\subsection{Ablation Studies}
	
	\subsubsection{Effect of Hyper-parameters $L_\nu$ and $L_\bh$} \label{subsection:ablation_study}
	In this experiment, we investigate the impact of the hyper-parameters $L_\nu$ and $L_\bh$ on the performance of the proposed method.
	To this end, we compare the performance with different hyper-parameters on Oxford-102 with $d{=}30$.
	From Table~\ref{tab:hyperparameter}, the performance deteriorates with the increase of $L_\nu$.
	In terms of $L_\bh$, we obtain the best performance with $L_h{=}1$.
	Thus, we set $L_\nu{=}0.0001$ and $L_\bh{=}1$ in practice.

	\begin{table*}[t]
		\normalsize
		\centering
		\caption{Ablation study on $ d_B $ and $M$ in terms of IS and FID  on Oxford-102. We set $d{=}30$ for all the experiments.}
		\resizebox{1\textwidth}{!}{
			\renewcommand{\arraystretch}{1.15}
			\begin{tabular}{c||cc|cc|cc|cc||cc|cc|cc|cc}	
				\hline %\toprule
				\multicolumn{1}{c||}{\multirow{3}[0]{*}{Methods}} & \multicolumn{8}{c||}{Setting $M=128$} & \multicolumn{8}{c}{Setting $d_B=100$} \\
				\cline{2-17}
				& \multicolumn{2}{c|}{$ d_B=50 $}  
				& \multicolumn{2}{c|}{$ d_B=100 $} 
				& \multicolumn{2}{c|}{$ d_B=200 $} 
				& \multicolumn{2}{c||}{$ d_B=400 $} 
				& \multicolumn{2}{c|}{$ M=64 $}  
				& \multicolumn{2}{c|}{$ M=128 $} 
				& \multicolumn{2}{c|}{$ M=256 $} 
				& \multicolumn{2}{c}{$ M=512 $}\\
				\cline{2-17}
				& \multicolumn{1}{c}{IS}  
				& \multicolumn{1}{c|}{FID}
				& \multicolumn{1}{c}{IS} 
				& \multicolumn{1}{c|}{FID}
				& \multicolumn{1}{c}{IS} 
				& \multicolumn{1}{c|}{FID}
				& \multicolumn{1}{c}{IS} 
				& \multicolumn{1}{c||}{FID}   
				& \multicolumn{1}{c}{IS} 
				& \multicolumn{1}{c|}{FID}
				& \multicolumn{1}{c}{IS} 
				& \multicolumn{1}{c|}{FID}
				& \multicolumn{1}{c}{IS} 
				& \multicolumn{1}{c|}{FID}
				& \multicolumn{1}{c}{IS} 
				& \multicolumn{1}{c}{FID}  \\
				\hline %\midrule
				LCCGAN & 2.895 & 131.5 & 3.003 & 61.9 & 3.104 & 66.4 & 3.246 & 61.6    & 2.937 & 99.3 & 3.003 & 61.9 & 3.148 & 94.0 & 3.152 & 93.6 \\
				LCCGAN++ & 2.673 & 124.8 & \textbf{3.370} & \textbf{57.7} & 3.362 & 62.8 & 3.276 & 62.0    & 3.131 & 66.3 & \textbf{3.370} & \textbf{57.7} & 3.068 & 76.8 & 3.211 & 63.3 \\ 
				\hline %\bottomrule
			\end{tabular}
		}
		\label{tab:ablation_study}
	\end{table*}
	
	\begin{table*}[t]
		\caption{
			Visual comparisons of the images produced by the models trained with different $d_B$ and $M$ on Oxford-102. 
		}
		\label{tab:ablation_dB_M}
		\centering
		\includegraphics[width=1\linewidth]{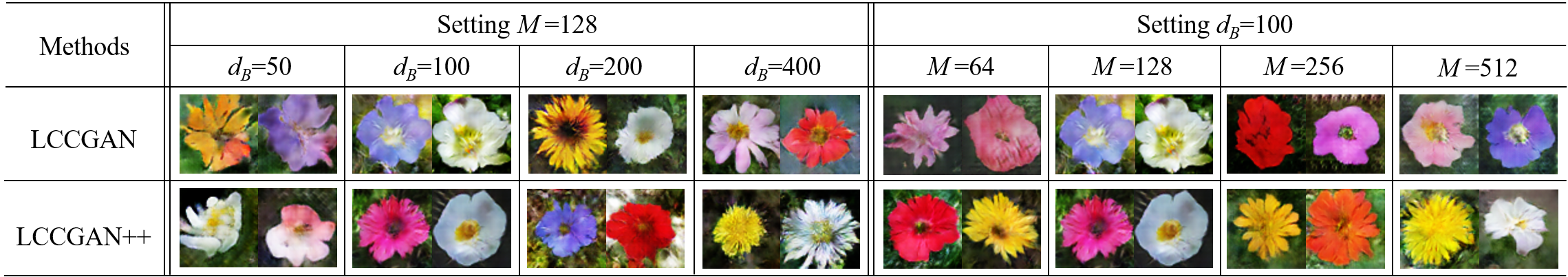} 
	\end{table*}
	
	\subsubsection{Effect of End-to-end Training}
	In this experiment, we compare the end-to-end training method with our multiple-stage training method. 
	% 	As shown in Algorithm \ref{alg:lccgan}, we adopt a multiple-stage method to train our models. Actually, similar to ALI~\cite{dumoulin2016adversarially} and BiGAN~\cite{donahue2016adversarial}, we can also train LCCGAN models in an end-to-end manner. Specifically, 
	In the end-to-end training scheme, we optimize a joint objective function by combining the loss of autoencoder, the objective of LCC, and the objective of a GAN model.
	From Table \ref{tab:lccgan_end2end}, the model with the multiple-stage strategy significantly outperforms the model with end-to-end manner.
	In contrast, the end-to-end training method may obtain inaccurate bases since it has to compensate for the objectives of autoencoder and GAN.
	With such inaccurate bases, the performance of LCCGAN would deteriorate.

	\subsubsection{Effect of the Ratio of $\#$class to $ d $}
	In this experiment, we investigate the ratio of the number of classes ($\#$class) to the number of local bases $d$. 
	Specifically, we fix $d$ to study the impact of the number of classes by varying $\#$class on Oxford-102 (containing 102 classes).  
	Note that with the increase of $\#$class, the number of training samples will increase accordingly. However, it would affect the performance of GANs.
	% 	To be specific, according to the generalization analysis, the more training samples $N$ we have, the better generalization performance of GANs would obtain. 
	To remove the influence of the number of training samples, we sample images from different classes and keep the total number of training samples fixed. 
	% 	By setting $d{=}3$ and $d{=}5$, we show the results in terms of IS and FID in Table~\ref{table:ratio1}. 
	
	We set $N$ to be the smallest number of training samples in the case of $\#$class${=}30$, \ie $N{=}1739$.
	From Table~\ref{table:ratio1}, when we increase the number of classes from $30$ to $102$, the data become more complicated and thus need more local bases to represent the manifold of data. 
	As a result, given a fixed number of local bases $d$, the images generated by the LCCGAN++ models tend to yield worse performance with the increase of $\#$class.

	\subsubsection{Effect of $ d_B $ and $M$}
	In this experiment, we conduct ablation studies to investigate the effect of the dimension of latent space ($ d_B $) and the number of bases ($ M $).
	% 	Specifically, we study the impact of $ d_B $ by setting $ M{=}128 $, and similarly study the impact of $ M $ by setting $ d_B{=}100 $.
	From Table \ref{tab:ablation_study}, when setting $d_B{=}100$ and $ M{=}128$, both LCCGAN and LCCGAN++ yield significantly better performance than the settings with a low dimension $d_B{=}50$ or a small number $M{=}64$. If we further increase $d_B$ and $M$, it would introduce additional computational cost but does not yield significant performance improvement.
	% 	Therefore, in practice, we set the dimension of latent space and the number of bases as $ d_B{=}100 $ and $ M{=}128 $, respectively.
	Furthermore, we also provide visual comparisons of the images produced by the models trained with different $ d_B $ and $ M $ in Table~\ref{tab:ablation_dB_M}.
	In practice, we set the dimension of latent space and the number of bases as $ d_B{=}100 $ and $ M{=}128 $, respectively.
	% 	From the results, LCCGAN {is able to} generate images with promising quality when $ d_B{=}100 $ and $ M{=}128 $. 

	\section{Conclusion} \label{sec:conclusion}
	We have proposed a novel generative model by using local coordinate coding (LCC) to improve the performance of GAN models. Unlike existing methods, we develop an LCC-based sampling method to exploit the local information on the latent manifold of real data. 
	Moreover, we also propose an advanced LCCGAN++ by introducing a higher-order term in the generator approximation.
	In this way, we are able
	to conduct analysis on the generalization performance of GANs
	and theoretically prove that a low-dimensional input is able to achieve good performance. 
	Qualitative and quantitative experiments on several benchmark datasets demonstrate the effectiveness of the proposed method over several baseline methods.

	\section*{Acknowledgments}
	This work was partially supported by the Key-Area Research and Development Program of Guangdong Province (2018B010107001), National Natural Science Foundation of China (NSFC) 61836003 (key project), Guangdong Project 2017ZT07X183, 
	Fundamental Research Funds for the Central Universities D2191240.

	%\ifCLASSOPTIONcompsoc
	%  % The Computer Society usually uses the plural form
	%\else
	%  % regular IEEE prefers the singular form
	%\fi

	\ifCLASSOPTIONcaptionsoff
	\newpage
	\fi

	
	
	\bibliographystyle{IEEEtran}
	\bibliography{main}

\begin{thebibliography}{10}
\providecommand{\url}[1]{#1}
\csname url@samestyle\endcsname
\providecommand{\newblock}{\relax}
\providecommand{\bibinfo}[2]{#2}
\providecommand{\BIBentrySTDinterwordspacing}{\spaceskip=0pt\relax}
\providecommand{\BIBentryALTinterwordstretchfactor}{4}
\providecommand{\BIBentryALTinterwordspacing}{\spaceskip=\fontdimen2\font plus
\BIBentryALTinterwordstretchfactor\fontdimen3\font minus
  \fontdimen4\font\relax}
\providecommand{\BIBforeignlanguage}[2]{{%
\expandafter\ifx\csname l@#1\endcsname\relax
\typeout{** WARNING: IEEEtran.bst: No hyphenation pattern has been}%
\typeout{** loaded for the language `#1'. Using the pattern for}%
\typeout{** the default language instead.}%
\else
\language=\csname l@#1\endcsname
\fi
#2}}
\providecommand{\BIBdecl}{\relax}
\BIBdecl

\bibitem{goodfellow2014gans}
I.~Goodfellow, J.~Pouget-Abadie, M.~Mirza, B.~Xu, D.~Warde-Farley, S.~Ozair,
  A.~Courville, and Y.~Bengio, ``Generative adversarial nets,'' in
  \emph{Advances in Neural Information Processing Systems}, 2014.

\bibitem{cao2018adversarial}
J.~Cao, Y.~Guo, Q.~Wu, C.~Shen, and M.~Tan, ``Adversarial learning with local
  coordinate coding,'' in \emph{Proceedings of the International Conference on
  Machine Learning}, 2018.

\bibitem{arjovsky2017wasserstein}
M.~Arjovsky, S.~Chintala, and L.~Bottou, ``Wasserstein generative adversarial
  networks,'' in \emph{Proceedings of the International Conference on Machine
  Learning}, 2017.

\bibitem{zhu2019dynamic}
M.~Zhu, P.~Pan, W.~Chen, and Y.~Yang, ``Dm-gan: Dynamic memory generative
  adversarial networks for text-to-image synthesis,'' in \emph{IEEE Conference
  on Computer Vision and Pattern Recognition}, 2019.

\bibitem{lin2019exploring}
J.~{Lin}, Z.~{Chen}, Y.~{Xia}, S.~{Liu}, T.~{Qin}, and J.~{Luo}, ``Exploring
  explicit domain supervision for latent space disentanglement in unpaired
  image-to-image translation,'' \emph{IEEE Transactions on Pattern Analysis and
  Machine Intelligence}, pp. 1--1, 2019.

\bibitem{mao2018on}
X.~{Mao}, Q.~{Li}, H.~{Xie}, R.~Y.~K. {Lau}, Z.~{Wang}, and S.~P. {Smolley},
  ``On the effectiveness of least squares generative adversarial networks,''
  \emph{IEEE Transactions on Pattern Analysis and Machine Intelligence},
  vol.~41, no.~12, pp. 2947--2960, 2019.

\bibitem{otverdout2020dynamic}
N.~{OTBERDOUT}, M.~{Daoudi}, A.~{Kacem}, L.~{Ballihi}, and S.~{Berretti},
  ``Dynamic facial expression generation on hilbert hypersphere with
  conditional wasserstein generative adversarial nets,'' \emph{IEEE
  Transactions on Pattern Analysis and Machine Intelligence}, pp. 1--1, 2020.

\bibitem{pan2020physics}
J.~{Pan}, J.~{Dong}, Y.~{Liu}, J.~{Zhang}, J.~{Ren}, J.~{Tang}, Y.~W. {Tai},
  and M.~{Yang}, ``Physics-based generative adversarial models for image
  restoration and beyond,'' \emph{IEEE Transactions on Pattern Analysis and
  Machine Intelligence}, pp. 1--1, 2020.

\bibitem{ranzato2014video}
M.~Ranzato, A.~Szlam, J.~Bruna, M.~Mathieu, R.~Collobert, and S.~Chopra,
  ``{V}ideo (language) {M}odeling: a {B}aseline for {G}enerative {M}odels of
  {N}atural {V}ideos,'' \emph{arXiv preprint arXiv:1412.6604}, 2014.

\bibitem{mathieu2015deep}
M.~Mathieu, C.~Couprie, and Y.~LeCun, ``{D}eep {M}ulti-scale {V}ideo
  {P}rediction beyond {M}ean {S}quare {E}rror,'' in \emph{International
  Conference on Learning Representations}, 2016.

\bibitem{cao2019multi}
J.~Cao, L.~Mo, Y.~Zhang, K.~Jia, C.~Shen, and M.~Tan, ``Multi-marginal
  wasserstein gan,'' in \emph{Advances in Neural Information Processing
  Systems}, 2019.

\bibitem{isola2017image}
P.~Isola, J.-Y. Zhu, T.~Zhou, and A.~A. Efros, ``Image-to-image translation
  with conditional adversarial networks,'' in \emph{IEEE International
  Conference on Computer Vision}, 2017.

\bibitem{kim2017learning}
T.~Kim, M.~Cha, H.~Kim, J.~Lee, and J.~Kim, ``Learning to discover cross-domain
  relations with generative adversarial networks,'' in \emph{Proceedings of the
  International Conference on Machine Learning}, 2017.

\bibitem{tzeng2017adversarial}
E.~Tzeng, J.~Hoffman, K.~Saenko, and T.~Darrell, ``Adversarial discriminative
  domain adaptation,'' in \emph{IEEE Conference on Computer Vision and Pattern
  Recognition}, 2017.

\bibitem{zhu2019simreal}
F.~Zhu, L.~Zhu, and Y.~Yang, ``Sim-real joint reinforcement transfer for 3d
  indoor navigation,'' in \emph{IEEE Conference on Computer Vision and Pattern
  Recognition}, 2019.

\bibitem{guo2020closed}
Y.~Guo, J.~Chen, J.~Wang, Q.~Chen, J.~Cao, Z.~Deng, Y.~Xu, and M.~Tan,
  ``Closed-loop matters: Dual regression networks for single image
  super-resolution,'' in \emph{IEEE Conference on Computer Vision and Pattern
  Recognition}, 2020.

\bibitem{radford2015unsupervised}
A.~Radford, L.~Metz, and S.~Chintala, ``Unsupervised representation learning
  with deep convolutional generative adversarial networks,'' \emph{arXiv
  preprint arXiv:1511.06434}, 2015.

\bibitem{karras2017progressive}
T.~Karras, T.~Aila, S.~Laine, and J.~Lehtinen, ``Progressive growing of gans
  for improved quality, stability, and variation,'' in \emph{International
  Conference on Learning Representations}, 2018.

\bibitem{guo2019auto}
Y.~Guo, Q.~Chen, J.~Chen, Q.~Wu, Q.~Shi, and M.~Tan, ``Auto-embedding
  generative adversarial networks for high resolution image synthesis,''
  \emph{IEEE Transactions on Multimedia}, vol.~21, no.~11, pp. 2726--2737,
  2019.

\bibitem{hinton2006reducing}
G.~E. Hinton and R.~R. Salakhutdinov, ``Reducing the dimensionality of data
  with neural networks,'' \emph{Science}, 2006.

\bibitem{tenenbaum2000global}
J.~B. Tenenbaum, V.~De~Silva, and J.~C. Langford, ``A global geometric
  framework for nonlinear dimensionality reduction,'' \emph{Science}, 2000.

\bibitem{roweis2000nonlinear}
S.~T. Roweis and L.~K. Saul, ``Nonlinear dimensionality reduction by locally
  linear embedding,'' \emph{Science}, 2000.

\bibitem{yu2009nonlinear}
K.~Yu, T.~Zhang, and Y.~Gong, ``Nonlinear learning using local coordinate
  coding,'' in \emph{Advances in Neural Information Processing Systems}, 2009.

\bibitem{gulrajani2017improved}
I.~Gulrajani, F.~Ahmed, M.~Arjovsky, V.~Dumoulin, and A.~C. Courville,
  ``Improved training of wasserstein gans,'' in \emph{Advances in Neural
  Information Processing Systems}, 2017.

\bibitem{qi2018global}
G.-J. Qi, L.~Zhang, H.~Hu, M.~Edraki, J.~Wang, and X.-S. Hua, ``Global versus
  localized generative adversarial nets,'' in \emph{IEEE Conference on Computer
  Vision and Pattern Recognition}, 2018.

\bibitem{kingma2013auto}
D.~P. Kingma and M.~Welling, ``Auto-encoding variational bayes,'' \emph{arXiv
  preprint arXiv:1312.6114}, 2013.

\bibitem{tolstikhin2018wasserstein}
I.~Tolstikhin, B.~Olivier, S.~Gelly, and B.~Schoelkopf, ``Wasserstein
  auto-encoders,'' in \emph{International Conference on Learning
  Representations}, 2018.

\bibitem{makhzani2015adversarial}
A.~Makhzani, J.~Shlens, N.~Jaitly, I.~Goodfellow, and B.~Frey, ``Adversarial
  autoencoders,'' \emph{arXiv preprint arXiv:1511.05644}, 2015.

\bibitem{zoph2016neural}
B.~Zoph and Q.~V. Le, ``Neural architecture search with reinforcement
  learning,'' \emph{arXiv preprint arXiv:1611.01578}, 2016.

\bibitem{guo2019nat}
Y.~Guo, Y.~Zheng, M.~Tan, Q.~Chen, J.~Chen, P.~Zhao, and J.~Huang, ``Nat:
  Neural architecture transformer for accurate and compact architectures,'' in
  \emph{Advances in Neural Information Processing Systems}, 2019.

\bibitem{gong2019autogan}
X.~Gong, S.~Chang, Y.~Jiang, and Z.~Wang, ``Autogan: Neural architecture search
  for generative adversarial networks,'' in \emph{Proceedings of the IEEE
  International Conference on Computer Vision}, 2019, pp. 3224--3234.

\bibitem{dziugaite2015training}
G.~K. Dziugaite, D.~M. Roy, and Z.~Ghahramani, ``Training generative neural
  networks via maximum mean discrepancy optimization,'' in \emph{Uncertainty in
  Artificial Intelligence}, 2015.

\bibitem{thanh-tung2018improving}
H.~Thanh-Tung, T.~Tran, and S.~Venkatesh, ``Improving generalization and
  stability of generative adversarial networks,'' in \emph{International
  Conference on Learning Representations}, 2019.

\bibitem{jiang2018on}
H.~Jiang, Z.~Chen, M.~Chen, F.~Liu, D.~Wang, and T.~Zhao, ``On computation and
  generalization of generative adversarial networks under spectrum control,''
  in \emph{International Conference on Learning Representations}, 2019.

\bibitem{arora2017gans}
S.~Arora, R.~Ge, Y.~Liang, T.~Ma, and Y.~Zhang, ``Generalization and
  equilibrium in generative adversarial nets ({GAN}s),'' in \emph{Proceedings
  of the International Conference on Machine Learning}, 2017.

\bibitem{zhang2018on}
P.~Zhang, Q.~Liu, D.~Zhou, T.~Xu, and X.~He, ``On the
  discrimination-generalization tradeoff in {GAN}s,'' in \emph{International
  Conference on Learning Representations}, 2018.

\bibitem{yu2010improved}
K.~Yu and T.~Zhang, ``Improved local coordinate coding using local tangents.''
  in \emph{Proceedings of the International Conference on Machine Learning},
  2010.

\bibitem{ulyanov2017adversarial}
D.~Ulyanov, A.~Vedaldi, and V.~Lempitsky, ``Adversarial generator-encoder
  networks,'' \emph{arXiv preprint arXiv:1704.02304}, 2017.

\bibitem{zhang2017stackgan}
H.~Zhang, T.~Xu, H.~Li, S.~Zhang, X.~Wang, X.~Huang, and D.~N. Metaxas,
  ``Stackgan: Text to photo-realistic image synthesis with stacked generative
  adversarial networks,'' in \emph{IEEE International Conference on Computer
  Vision}, 2017.

\bibitem{lecun1998gradient}
Y.~LeCun, L.~Bottou, Y.~Bengio, and P.~Haffner, ``Gradient-based learning
  applied to document recognition,'' \emph{Proceedings of the IEEE}, 1998.

\bibitem{nilsback2008automated}
M.-E. Nilsback and A.~Zisserman, ``Automated flower classification over a large
  number of classes,'' in \emph{Indian Conference on Computer Vision, Graphics
  and Image Processing}, 2008.

\bibitem{song2015construction}
F.~Yu, A.~Seff, Y.~Zhang, S.~Song, T.~Funkhouser, and J.~Xiao, ``Construction
  of a large-scale image dataset using deep learning with humans in the loop,''
  \emph{arXiv preprint arXiv:1506.03365}, 2015.

\bibitem{liu2015deep}
Z.~Liu, P.~Luo, X.~Wang, and X.~Tang, ``Deep learning face attributes in the
  wild,'' in \emph{IEEE International Conference on Computer Vision}, 2015.

\bibitem{deng2009imagenet}
J.~Deng, W.~Dong, R.~Socher, L.-J. Li, K.~Li, and L.~Fei-Fei, ``Imagenet: A
  large-scale hierarchical image database,'' in \emph{IEEE Conference on
  Computer Vision and Pattern Recognition}, 2009, pp. 248--255.

\bibitem{salimans2016improved}
T.~Salimans, I.~Goodfellow, W.~Zaremba, V.~Cheung, A.~Radford, and X.~Chen,
  ``Improved techniques for training gans,'' in \emph{Advances in Neural
  Information Processing Systems}, 2016.

\bibitem{heusel2017gans}
M.~Heusel, H.~Ramsauer, T.~Unterthiner, B.~Nessler, and S.~Hochreiter, ``Gans
  trained by a two time-scale update rule converge to a local nash
  equilibrium,'' in \emph{Advances in Neural Information Processing Systems},
  2017.

\bibitem{miyato2018cgans}
T.~Miyato and M.~Koyama, ``c{GAN}s with projection discriminator,'' in
  \emph{International Conference on Learning Representations}, 2018.

\bibitem{zhang2018stackgan++}
H.~Zhang, T.~Xu, H.~Li, S.~Zhang, X.~Wang, X.~Huang, and D.~N. Metaxas,
  ``Stackgan++: Realistic image synthesis with stacked generative adversarial
  networks,'' \emph{IEEE Transactions on Pattern Analysis and Machine
  Intelligence}, vol.~41, no.~8, pp. 1947--1962, 2018.

\bibitem{bach2017breaking}
F.~Bach, ``Breaking the curse of dimensionality with convex neural networks,''
  \emph{Journal of Machine Learning Research}, vol.~18, no.~1, pp. 629--681,
  2017.

\bibitem{ledoux2013probability}
M.~Ledoux and M.~Talagrand, \emph{Probability in Banach Spaces: isoperimetry
  and processes}.\hskip 1em plus 0.5em minus 0.4em\relax Springer Science \&
  Business Media, 2013.

\bibitem{hsu2009multi}
D.~J. Hsu, S.~M. Kakade, J.~Langford, and T.~Zhang, ``Multi-label prediction
  via compressed sensing,'' in \emph{Advances in Neural information processing
  systems}, 2009, pp. 772--780.

\bibitem{kingma2014adam}
D.~P. Kingma and J.~Ba, ``Adam: A method for stochastic optimization,'' in
  \emph{International Conference on Learning Representations}, 2015.

\bibitem{he2015delving}
K.~He, X.~Zhang, S.~Ren, and J.~Sun, ``Delving deep into rectifiers: Surpassing
  human-level performance on imagenet classification,'' in \emph{IEEE
  International Conference on Computer Vision}, 2015.

\end{thebibliography}

	\newpage
	
	\onecolumn
	\begin{center}
		~\\
		\Huge{Supplementary Materials} \\ 
		\tiny{~}\\ 
		\LARGE{Improving Generative Adversarial Networks with Local Coordinate Coding}\\
		\tiny{~}\\
		\large{
		Jiezhang Cao$^*$,
		Yong Guo$^*$,
		Qingyao Wu,
		Chunhua Shen,
		Junzhou Huang,
		Mingkui Tan$^\dagger$
		}
	\end{center}
	
	\setcounter{section}{0}
	\renewcommand\thesection{\Alph{section}}
	~\\
	In the supplementary materials, we provide detailed proofs for all lemmas, theorems and corollary. 
	Besides, we give more experiment settings and results. 
	We organize our supplementary materials as follows.
	In Sections \ref{sec:proofs_lamma1} and \ref{sec:proofs_lamma2}, we give the proofs of the generator approximation and its improved version, respectively. 
	In Sections \ref{sec:thm3}, \ref{sec:thm1} and \ref{sec:thm2}, we provide the generalization analysis for our method.
	In Section \ref{sec:detail}, we provide more experimental details. 
	In Section \ref{sec:more_res}, we provide more results of our proposed method. 
	
	\section{Proofs of Lemma \ref{lemma1}} \label{sec:proofs_lamma1}
	Based on \cite{yu2009nonlinear, cao2018adversarial}, we first use the definition of Lipschitz smoothness as follows. 
	\begin{deftn} \cite{yu2009nonlinear}
		A function $ f_{\theta}(\bx) $ in $ \mmR^d $ is $ (L_{\bx}, L_{f}) $-Lipschitz smooth if $ \| f(\bx') - f(\bx) \|_2 \leq L_{\bx} \| \bx - \bx' \|_2 $ and $ \| f(\bx') - f(\bx) - \nabla f(\bx)^\trsp (\bx' - \bx) \|_2 \leq L_{f} \| \bx - \bx' \|_2^2 $, where $ L_{\bx}, L_{f} > 0 $.
	\end{deftn}
	Using this definition, we then provide the following proposition to complete the proofs of the generator approximation.
	\begin{prop} \label{lemma: DG Approximation}
		Let $ (\bgamma, \mC) $ be an arbitrary coordinate coding on $ \mmR^{d_B} $.
		Given an $ (L_{\bh}, L_{G}) $-Lipschitz smooth generator $ G_u(\bh) $ and an $ L_{\bx} $-Lipschitz discriminator $ D_{v} $, for all $ \bh \in \mmR^{d_B} $:
		\begin{align}
		\left| {D}_v(G_u(\bh)) - {D}_v \left(\sum_{\bv} \gamma_{\bv} (\bh) G_u(\bv) \right) \right| 
		\leq L_{\bx} L_{\bh} \| \bh - \br(\bh) \| + L_{\bx} L_G \sum_{\bv \in \mC} |\gamma_{\bv} (\bh)| \| \bv - \br(\bh) \|^{2}.
		\end{align}
	\end{prop}
	\begin{proof}
		Given an $ (L_{\bh}, L_{G}) $-Lipschitz smooth generator $ G_u(\bh) $, an $ L_{\bx} $-Lipschitz discriminator $ D_{v} $, and let $ \gamma_{\bv} = \gamma_{\bv}(\bh) $ and $ \bh' = \br(\bh) = \sum_{\bv \in \mC} \gamma_{\bv} \bv $. We have
		\begin{equation}
		\begin{aligned}
		&\left| \widetilde{D}_v(G_u(\bh)) - \widetilde{D}_v \left(\sum_{\bv} \gamma_{\bv} (\bh) G_u(\bv) \right) \right| \\
		=& \left| {D}_v(G_u(\bh)) - {D}_v \left(\sum_{\bv} \gamma_{\bv} (\bh) G_u(\bv) \right) \right| \\
		=& \left| {D}_v(G_u(\bh)) - {D}_v(G_u(\bh')) - \left( {D}_v \left(\sum_{\bv} \gamma_{\bv} (\bh) G_u(\bv) \right) - {D}_v(G_u(\bh')) \right)  \right| \\
		\leq& \left| {D}_v \left( G_u(\bh) \right) - {D}_v \left( G_u(\bh') \right) \right| + \left| {D}_v \left(\sum_{\bv} \gamma_{\bv} (\bh) G_u(\bv) \right) - {D}_v \left( G_u(\bh') \right)  \right| \\
		\leq& L_{\bx} \left\| G_u(\bh) - G_u(\bh') \right\| + L_{\bx} \left\| \sum_{\bv} \gamma_{\bv} (\bh) G_u(\bv) - G_u(\bh') \right\| \\
		\leq& L_{\bx} \left\| G_u(\bh) - G_u(\bh') \right\| + L_{\bx} \left\| \sum_{\bv} \gamma_{\bv} (\bh) \left( G_u(\bv) - G_u(\bh') - \Delta G_u (\bh')^{\trsp} \left( \bv - \bh' \right) \right)  \right\| \\
		\leq& L_{\bx} \left\| G_u(\bh) - G_u(\bh') \right\| + L_{\bx} \sum_{\bv \in \mC} |\gamma_{\bv}| \left\| G_u (\bv) - G_u(\bh') - \Delta G_u(\bh')^{\trsp} (\bv - \bh') \right\| \\
		\leq& L_{\bx} L_{\bh} \| \bh - \bh' \| + L_{\bx} L_{G} \sum_{\bv \in \mC} |\gamma_{\bv} | \| \bv - \bh' \|^2 \\
		=& L_{\bx} L_{\bh} \| \bh -  \br(\bh) \| + L_{\bx} L_{G} \sum_{\bv \in \mC} |\gamma_{\bv} | \| \bv - \br(\bh) \|^2,
		\end{aligned}
		\end{equation}
		where $ \widetilde{D}_v(\cdot) = 1 - D_v(\cdot) $. In the above derivation, the first inequality holds by the triangle inequality. The second inequality uses an assumption that $ D_v $ is Lipschitz smooth \wrt the input. The third inequality uses the facts that $ \sum_{\bv \in \mC} \gamma_{\bv} (\bx) = 1 $ and $ \bh' = \sum_{\bv \in \mC} \gamma_{\bv} \bv $. The last inequality uses the $ (L_{\bh}, L_{G}) $-Lipschitz smooth generator $ G_u $, that is
		\begin{align}
		\left\| G_u (\bv) - G_u(\bh') - \Delta G_u(\bh')^{\trsp} (\bv - \bh') \right\| \leq L_{G} \| \bv - \bh' \|^2.
		\end{align}
	\end{proof}

	% \section{Proof of Lemma \ref{lemma: Generator_Approximation_1}}
	\begin{lemma} \textbf{\emph{(Generator approximation) }} \label{lemma1}
		Let $ \br(\bh) = \sum_{\bv \in \mC} \gamma_{\bv}(\bh) \bv $, and $ (\bgamma, \mC) $ be an arbitrary coordinate coding on $ \mmR^{d_B} $.
		Given a Lipschitz smooth generator $ G_u(\bh) $, for all $ \bh \in \mmR^{d_B} $:
		\begin{align}
		\left\| G_u\left(\sum_{\bv \in \mC} \gamma_{\bv}(\bh) \bv\right) - \sum_{\bv \in \mC} \gamma_{\bv} (\bh) G_u(\bv) \right\| 
		\leq 2L_{\bh} \| \bh - \br(\bh) \| + L_G \sum_{\bv \in \mC} |\gamma_{\bv} (\bh)| \| \bv - \br(\bh) \|^{2}.
		\end{align}
	\end{lemma}
	
	\begin{proof}
		From Lemma \ref{lemma: DG Approximation}, when the discriminator is identity function: $ D_v(t) = t $, that is
		\begin{equation}
		\begin{aligned}
		\left| {D}_v(G_u(\bh)) - {D}_v \left(\sum_{\bv} \gamma_{\bv} (\bh) G_u(\bv) \right) \right| &=
		\left\| G_u(\bh) - \sum_{\bv} \gamma_{\bv} (\bh) G_u(\bv) \right\| \\
		&\leq L_{\bh} \| \bh -  \br(\bh) \| + L_{G} \sum_{\bv \in \mC} |\gamma_{\bv} | \| \bv - \br(\bh) \|^2,
		\end{aligned}
		\end{equation}
		then, we have
		\begin{equation}
		\begin{aligned}
		\left\| G_u \left(\sum_{\bv \in \mC} \gamma_{\bv}(\bh) \bv\right) - \sum_{\bv \in \mC} \gamma_{\bv} (\bh) G_u(\bv) \right\|
		&= \left\| G_u \left(\sum_{\bv \in \mC} \gamma_{\bv}(\bh) \bv\right) - G_u\left(\bh\right) + G_u\left(\bh\right) - \sum_{\bv \in \mC} \gamma_{\bv} (\bh) G_u(\bv) \right\| \\
		&\leq \left\| G_u \left(\sum_{\bv \in \mC} \gamma_{\bv}(\bh) \bv\right) - G_u\left(\bh\right) \right\| + \left\| G_u\left(\bh\right) - \sum_{\bv \in \mC} \gamma_{\bv} (\bh) G_u(\bv) \right\| \\
		&\leq 2L_{\bh} \| \bh - \br(\bh) \| + L_G \sum_{\bv \in \mC} |\gamma_{\bv} (\bh)| \| \bv - \br(\bh) \|^{2},
		\end{aligned}
		\end{equation}
		where $ \br(\bh) = \sum_{\bv \in \mC} \gamma_{\bv}(\bh) \bv $.
	\end{proof}

	\section{Proofs of Lemma \ref{lemma2}} \label{sec:proofs_lamma2}
	\begin{lemma}\label{lemma2} 
		\textbf{\emph{(Improved generator approximation) }} 
		Let $ \br(\bh) = \sum_{\bv \in \mC} \gamma_{\bv}(\bh) \bv $, and $ (\bgamma, \mC) $ be an arbitrary coordinate coding on $ \mmR^{d_B} $.
		Given a $ (L_{\bh}, L_{\nu}) $-Lipschitz smooth generator $ G_{u}(\bh) $, for all $ \bh \in \mmR^{d_B} $:
		\begin{equation}
		\begin{aligned}%\label{eqn:ext_G_appro}
		\left\| G_{u}\left(\br(\bh)\right) {-} \sum\limits_{\bv \in \mC} \gamma_{\bv}(\bh) \left( G_{u}(\bv) {+} \frac{1}{2} \nabla G_{u}(\bv)^{\trsp} (\bh {-} \bv) \right)  \right\|  
		\leq 2 L_{\bh} \| \bh {-} \br(\bh) \| {+} L_{\nu} \sum_{\bv \in \mC} |\gamma_{\bv} (\bh)| {\cdot} \| \bv {-} \br(\bh) \|^{3},
		\end{aligned}
		\end{equation}
		where $ \br(\bh) {=} \sum\limits_{\bv {\in} \mC} \gamma_{\bv} (\bh) \bv $.
	\end{lemma}
	\begin{proof}
		Let $\bh' = \br(\bh) $, we have
		\begin{equation}
		\begin{aligned}
		&\left\| G_{u}\left(\br(\bh)\right) - G_{u}(\bh) + G_{u}(\bh) - \sum\limits_{\bv \in \mC} \gamma_{\bv}(\bh) \left( G_{u}(\bv) + \frac{1}{2} \nabla G_{u}(\bv)^{\trsp} (\bh - \bv) \right)  \right\| \\
		\leq& \left\| G_{u}\left(\br(\bh)\right) - G_{u}(\bh) \right\| + \left\| G_{u}(\bh) - \sum\limits_{\bv \in \mC} \gamma_{\bv}(\bh) \left( G_{u}(\bv) + \frac{1}{2} \nabla G_{u}(\bv)^{\trsp} (\bh - \bv) \right)  \right\|  \\
		=& L_{\bh} \left\| \bh - \bh' \right\| + \left\| \sum\limits_{\bv \in \mC} \gamma_{\bv}(\bh) \left( G_{u}(\bv) - G_{u}(\bh) - \frac{1}{2} \nabla G_{u}(\bv)^{\trsp} (\bv - \bh') + \frac{1}{2} \nabla G_{u}(\bv)^{\trsp} (\bh - \bv)  \right) \right\| \\
		\leq& L_{\bh} \left\| \bh - \bh' \right\| + \frac{1}{2} \| \nabla G_u (\bh)^{\trsp} (\bh - \bh') \| + \left\| \sum\limits_{\bv \in \mC} \gamma_{\bv}(\bh) \left( G_{u}(\bv) - G_{u}(\bh) - \frac{1}{2} \left(\nabla G_{u}(\bh) + \nabla G_{u}(\bv)\right)^{\trsp} (\bv - \bh)  \right) \right\| \\
		\leq& 2 L_{\bh} \left\| \bh - \bh' \right\| + \left\| \sum\limits_{\bv \in \mC} \gamma_{\bv}(\bh) \left( G_{u}(\bv) - G_{u}(\bh) - \frac{1}{2} \left(\nabla G_{u}(\bh) + \nabla G_{u}(\bv)\right)^{\trsp} (\bv - \bh)  \right) \right\| \\
		\leq& 2 L_{\bh} \left\| \bh - \bh' \right\| + L_{\nu} \sum\limits_{\bv \in \mC} \gamma_{\bv}(\bh) \| \bh - \bv \|^3.
		\end{aligned}
		\end{equation}
	\end{proof}
	
	\newpage
	\section{Proof of Theorem \ref{theorem: Further Generalization Bound}} \label{sec:thm3}
	First, we introduce the following definition to measure the locality of a coding in LCCGAN++.
	\begin{deftn}~\textbf{\emph{(Localization measure)}} \label{def:local_measure_v2}
		Given $ L_{\bh}, L_{G} $, and coding $ (\bgamma, \mC) $, we define the localization measure $Q_{L_{\bh}, L_{G}}(\bgamma, \mC)$ as
		\begin{align}
		Q_{L_{\bh}, L_{\nu}}(\bgamma, \mC) = 2L_{\bh} \| \bh {-} \br(\bh) \| {+} L_{\nu} \sum_{\bv \in \mC} |\gamma_{\bv} (\bh)| {\cdot} \| \bv {-} \br(\bh) \|^{3}.
		\end{align}
	\end{deftn}
	
	When the latent points lie on a latent manifold and the generator is Lipschitz smooth, we slightly extend Lemma \ref{lemma: manifold_coding} based on \cite{yu2009nonlinear}.
	Then, $ Q_{L_{\bh}, L_{\nu}} (\bgamma, \mC) $ has a bound as follows.
	
	\begin{lemma}%\textbf{\emph{(Manifold Coding)}}
		\label{lemma: manifold_coding_v2}
		If the latent points lie on a compact smooth manifold $ \mM $, given an $ (L_{\bh}, L_{\nu}) $-Lipschitz smooth generator $ G_{u}(\bh) $ and any $ \epsilon > 0 $, then there exist anchor points $ \mC \subset \mM $ and coding $ \bgamma $ such that %$ |\mC| \leq (1 + d_{\mM}) \mN (\epsilon, \mM) $ and
		\begin{align}
		Q_{L_{\bh}, L_{\nu}} (\bgamma, \mC)
		\leq \left[ 2 L_{\bh}  c_{\mM} + \left(1 + \sqrt{d_{\mM}} + 8 \sqrt{d_{\mM}} \right) L_{\nu} \right] \epsilon^{3},
		\end{align}
		where $ d_{\mM} $ is the dimension of the latent manifold.
	\end{lemma}
	%\begin{*lemma}{\textbf{\emph{\ref{lemma: manifold_coding_v2}} }} 
	%	If the latent points lie on a compact smooth manifold $ \mM $, given an $ (L_{\bh}, L_{\nu}) $-Lipschitz smooth generator $ G_{u}(\bh) $ and any $ \epsilon > 0 $, then there exist anchor points $ \mC \subset \mM $ and coding $ \bgamma $ such that %$ |\mC| \leq (1 + d_{\mM}) \mN (\epsilon, \mM) $ and
	%	\begin{align*}
	%	Q_{L_{\bh}, L_{\nu}} (\bgamma, \mC)
	%	\leq \left[ 2 L_{\bh}  c_{\mM} + \left(1 + \sqrt{d_{\mM}} + 8 \sqrt{d_{\mM}} \right) L_{\nu} \right] \epsilon^{3},
	%	\end{align*}
	%	where $ d_{\mM} $ is the dimension of the latent manifold.
	%\end{*lemma}
	\begin{proof}
		Using the conclusion of \cite{yu2009nonlinear}, we directly have this lemma.
	\end{proof}
	In Lemma \ref{lemma: manifold_coding_v2}, the complexity of the local coordinate coding depends on the intrinsic dimension of the latent manifold instead of the dimension of the basis.
	
	\noindent\textbf{Theorem \ref{theorem: Further Generalization Bound}} 
	\emph{
		Suppose that $ \phi(\cdot) $ is Lipschitz smooth, and bounded in $ [-\Delta, \Delta] $. Given an sample set $ \mH $ in the latent space and an empirical distribution $ \widehat{\mD}_{\emph{\real}} $ with $ N $ samples drawn from $ \mD_{\emph{\real}} $, the following inequation holds with probability at least $ 1 - \delta $,
		\begin{align}
		\left| \mmE_{\mH} \left[d_{\mF, \phi} \left(\widehat{\mD}_{G_{\widehat{w}}}, {\mD}_{\emph{real}} \right) \right]
		{-} \inf_{ \mG } \mmE_{\mH} \left[ d_{\mF, \phi} \left({\mD}_{G_u},  {\mD}_{\emph{real}} \right) \right] \right| 
		\leq 2 {R}_{\mX}(\mF) + 2 \Delta \sqrt{\frac{2}{N} \log\left(\frac{1}{\delta}\right)} + 2\epsilon(d_{\mM}),
		\end{align}
		where $ {R}_{\mX}(\mF) $ is the Rademacher complexity of $\mF $ and {$ \epsilon(d_{\mM}) = L_{\phi} Q_{L_{\bh}, L_{\nu}} (\bgamma, \mC) + 2\Delta $}. 
	}

	\noindent
	\begin{proof}
		Based on Theorem \ref{theorem: generalization_LCCGAN_v2} and Lemma \ref{lemma: manifold_coding_v2}, we directly finish the proof.
	\end{proof}
	
	\noindent\textbf{Corollary \ref{coll:generalization_relu_D}} 
	\emph{
		Let $ \mX $ be the unit ball of $ \mmR^d $ under the $ \ell_2 $-norm, \ie $ \mX {=} \{ \bx {\in} \mmR^d{:\;} \| \bx \| {\leq} 1 \} $.
		Assume that the discriminator set $ \mF $ is the set of neural networks with a rectified linear unit (ReLU), 
		\[\mF {=} \left\{ \max \{ \bw^{\trsp}[\bx; 1], 0\}: \bw {\in} \mmR^{d+1}, \| \bw \| {=} 1 \right\},\] 
		then with probability at least $ 1-\delta $,
		\begin{align}
		\left| \mmE_{\mH} \left[d_{\mF, \phi} \left(\widehat{\mD}_{G_{\widehat{w}}}, {\mD}_{\emph{\real}} \right) \right]			{-} \inf_{ \mG } \mmE_{\mH} \left[ d_{\mF, \phi} \left({\mD}_{G_{u}},  {\mD}_{\emph{\real}} \right) \right] \right| 
		\leq 2 \Delta \sqrt{\frac{2}{N} \log \left(\frac{1}{\delta}\right)} + \frac{4\sqrt{2}}{\sqrt{N}} + 2\epsilon(d_{\mM}).
		\end{align}
		%	where $ {R}_{\mX}(\mF) $ is the Rademacher complexity of $\mF $.
	}

	\noindent
	\begin{proof}
		Part of the proof is from \cite{bach2017breaking, zhang2018on}. 
		% \begin{align*}
		%     \mF = \left\{ \max (\bw^{\trsp} [\bx; 1], 0): \bw \in \mmR^{d+1}, \|\bw\|=1 \right\}
		% \end{align*}
		Based on the definition of Rademacher complexity, we first estimate $ R_{\mX} (\mF) $ as follows,
		% \begin{align*}
		%     \sup_{D_v \in \mF} \left| \frac{2}{N} \sum_i \sigma_i D_v(\bx_i) \right| = \sup_{\|\bw\|=1} \left| \frac{2}{N} \sum_i \sigma_i \max (\bw^{\trsp} [\bx_i; 1], 0) \right|.
		% \end{align*}
		\begin{equation}
		\begin{aligned}
		R_{\mX} (\mF) =& \mmE_{\bsigma} \left[ \sup_{\|\bw\|=1} \left| \frac{2}{N} \sum_i \sigma_i \max(\bw^{\trsp} [\bx_i; 1], 0) \right|  \right] \\
		\leq& \mmE_{\bsigma} \left[ \sup_{\|\bw\|=1} \left| \frac{2}{N} \sum_i \sigma_i \bw^{\trsp} [\bx_i; 1] \right|  \right] \\
		=& \frac{2}{N} \mmE_{\bsigma} \left[ \left\| \sum_i \sigma_i [\bx_i, 1] \right\| \right] \\
		\leq& \frac{2 \sqrt{2}}{N}.
		\end{aligned}
		\end{equation}
		The second line uses the $1$-Lipschitz property of $\max(x, 0)$ and the third line follows by Talagrand's contraction lemma \cite{ledoux2013probability}.
		The last line holds by the Rademacher complexity of linear functions \cite{hsu2009multi}. 
		Then, we use this inequality and Theorem \ref{theorem: Further Generalization Bound} to to prove the result.
	\end{proof}

	\newpage
	\section{Proof of Theorem \ref{theorem: supp_Generalization Bound}} \label{sec:thm1}
	First, we introduce the following definition to measure the locality of a coding in LCCGAN \cite{cao2018adversarial}. %, and then develop our optimization problem.
	\begin{deftn}~\textbf{\emph{(Localization measure)}} \label{def:local_measure}
		Given $ L_{\bh}, L_{G} $, and coding $ (\bgamma, \mC) $, we define the localization measure $Q_{L_{\bh}, L_{G}}(\bgamma, \mC)$ as
		\begin{align}
		Q_{L_{\bh}, L_{G}}(\bgamma, \mC) = 2L_{\bh} \| \bh {-} \br(\bh) \| {+} L_{G} \sum_{\bv \in \mC} |\gamma_{\bv} (\bh)| {\cdot} \| \bv {-} \br(\bh) \|^{2}.
		\end{align}
	\end{deftn}
	%Following the setting of \cite{yu2009nonlinear}, we set $ L_{\bh}{=}0.25 $ and $ L_{G}{=}2 $ in practice.
	%By minimizing the localization quality, we will propose an objective function of the proposed method.
	
	%In order to provide a generalization bound \wrt the neural net distance, we first give some relevant lemmas and theorems.
	%Recall Lemma \ref{lemma: Generator Approximation}, the quality of generator approximation is determine by the Localization term $ \widehat{Q}_{L_{\bh}, L_{G}} (\bgamma, \mC) $.
	When the latent points lie on a latent manifold and the generator is Lipschitz smooth,  $ Q_{L_{\bh}, L_{G}} (\bgamma, \mC) $ has a bound as follows.
	\begin{lemma}\textbf{\emph{(Manifold coding \cite{yu2009nonlinear})}}
		\label{lemma: manifold_coding}
		If the latent points lie on a compact smooth manifold $ \mM $, given an $ (L_{\bh}, L_{G}) $-Lipschitz smooth generator $ G_u(\bh) $ and any $ \epsilon > 0 $, then there exist anchor points $ \mC \subset \mM $ and coding $ \bgamma $ such that %$ |\mC| \leq (1 + d_{\mM}) \mN (\epsilon, \mM) $ and
		\begin{align}
		Q_{L_{\bh}, L_{G}} (\bgamma, \mC)
		\leq \left[ 2L_{\bh}  c_{\mM} + \left(1 + \sqrt{d_{\mM}} + 4 \sqrt{d_{\mM}} \right) L_G \right] \epsilon^{2}.
		\end{align}
	\end{lemma}
	
	This lemma shows that the complexity of LCC coding depends on the intrinsic dimension of the manifold instead of the basis.
	Based on Lemma \ref{lemma: manifold_coding}, we have the following generalization bound on $\hat{\mD}_{\real}$ to develop the generalization analysis of LCCGAN.
	
	\begin{thm}\label{theorem: supp_Generalization Bound}
		Suppose measuring function $ \phi(\cdot) $ is Lipschitz smooth: $ | \phi'(\cdot) | \leq L_{\phi} $, and bounded in $ [-\Delta, \Delta] $. Consider coordinate coding $ (\bgamma, \mC) $, an example set $ \mH $ in latent space and the empirical distribution $ \widehat{\mD}_{\emph{real}} $, if the generator is Lipschitz smooth, then the expected generalization error satisfies the inequality:
		\begin{align}
		\mmE_{\mH} \left[ d_{\mF, \phi} \left(\widehat{\mD}_{G_{\widehat{w}}\left( \bgamma(\bh) \right)}, \widehat{\mD}_{\emph{real}} \right)
		\right] 
		\leq \inf_{ \mG } \mmE_{\mH} \left[ d_{\mF, \phi} \left( {\mD}_{ G_{u} (\bh)}, \widehat{\mD}_{\emph{real}} \right) \right] + \epsilon(d_{\mM}),
		\end{align}
		where $ \epsilon(d_{\mM}) = L_{\phi} Q_{L_{\bh}, L_{G}} (\bgamma, \mC) + 2\Delta $, and generative quality $ Q_{L_{\bh}, L_{G}} (\bgamma, \mC) $ has an upper bound \wrt $ d_{\mM} $ in Lemma \ref{lemma: manifold_coding}.
	\end{thm}
	
	\begin{proof}
		Let $ \mH^{(k)} = \left\{ \bh_1^{(k)}, \bh_2^{(k)}, \ldots, \bh_r^{(k)} \right\} $ be a set of $ r $ latent samples which lie on the latent distribution. Consider $ n+1 $ independent experiments over the latent distribution, we have $ {\mH}_{r, n+1} = \left\{ \mH^{(1)}, \mH^{(2)}, \ldots, \mH^{(n+1)} \right\} $. Recall the optimization problem, we consider an empirical version of the expected loss:
		\begin{align}\label{optimization}
		[\widetilde{w}] = \argmin_{[w]} \left[ \frac{1}{n} \sum_{i=1}^{n+1} d_{\mF, \phi} \left({\mD}_{G_{w, \mH^{(i)}} (\bgamma(\bh))}, \widehat{\mD}_{\real} \right) \right].
		\end{align}
		
		Let $ k $ be an integer randomly drawn from $ \{1, 2, \ldots, n+1\} $. Let $ \left[ \widehat{w}^{(k)} \right] $ be the solution of
		\begin{align}
		\left[\widehat{w}^{(k)}\right] = \argmin_{[w]} \left[ \frac{1}{n} \sum_{i\neq k}^{n+1} d_{\mF, \phi} \left({\mD}_{G_{w, \mH^{(i)}} (\bgamma(\bh))}, \widehat{\mD}_{\real} \right) \right],
		\end{align}
		with the $ k $-th example left-out.
		
		Recall the definition of the neural net distance, we have
		\begin{align*}
		d_{\mF, \phi} (\mu, \nu) = \sup\limits_{\mF} \left| \mathop \mmE\limits_{\bx \sim \mu} \left[ \phi(D_v(\bx)) \right] + \mathop \mmE\limits_{\bx \sim \nu} \left[ \phi(\widetilde{D}_v (\bx)) \right] \right|,
		\end{align*}
		where $ \mF = \{ D_v, v \in \mV \} $.
		Given the $ k $-th sample experiment, the same real distribution $ \widehat{\mD}_{\real} $ over the training samples $ \bx_1, \bx_2, \ldots, \bx_m $, and two different distributions generated by $ {G_{\widehat{w}^{(k)}, \mH^{(k)}}\left(\bgamma(\bh)\right)} $ and $ {G_{\widetilde{w}, \mH^{(k)}}\left(\bgamma(\bh)\right)} $, respectively, the difference value of the neural net distance between these two generated distributions is:
		\begin{equation}
		\begin{aligned}
		&d_{\mF, \phi} \left( \widehat{\mD}_{{G_{\widehat{w}^{(k)}, \mH^{(k)}}\left(\bgamma(\bh)\right)}}, \widehat{\mD}_{\real} \right)
		- d_{\mF, \phi} \left(\widehat{\mD}_{{G_{\widetilde{w}, \mH^{(k)}}\left(\bgamma(\bh)\right)}}, \widehat{\mD}_{\real} \right) \\
		=&\sup\limits \left| \mathop \mmE\nolimits_{\bx \in \widehat{\mD}_{\real}} \left[ \phi(D_v(\bx)) \right] + \mathop \mmE\nolimits_{\bh \in \mH^{(k)}} \left[ \phi\left(\widetilde{D}_v \left( {G_{\widehat{w}^{(k)}, \mH^{(k)}}\left(\bgamma(\bh)\right)} \right)\right) \right] \right| \\
		&- \sup\limits \left| \mathop \mmE\nolimits_{\bx \in \widehat{\mD}_{\real}} \left[ \phi(D_v(\bx)) \right] + \mathop \mmE\nolimits_{\bh \in \mH^{(k)}} \left[ \phi\left(\widetilde{D}_v \left( {G_{\widetilde{w}, \mH^{(k)}}\left(\bgamma(\bh)\right)} \right)\right) \right] \right|\\
		\leq& \sup \left| \mmE_{ \bh \in \mH^{(k)} } \left[ \phi \left( \widetilde{D}_{v} \left( {G_{\widehat{w}^{(k)}, \mH^{(k)}}\left(\bgamma(\bh)\right)} \right) \right) \right]
		- \mmE_{ \bh \in \mH^{(k)} } \left[ \phi \left( \widetilde{D}_{v} \left( {G_{\widetilde{w}, \mH^{(k)}}\left(\bgamma(\bh)\right)} \right) \right) \right] \right| \\
		=& \sup \left| \frac{1}{\left| \mH^{(k)} \right|} \sum\limits_{\bh \in \mH^{(k)}} \left[ \phi \left( \widetilde{D}_{v} \left( {G_{\widehat{w}^{(k)}, \mH^{(k)}}\left(\bgamma(\bh)\right)} \right) \right)
		- \phi \left( \widetilde{D}_{v} \left( {G_{\widetilde{w}, \mH^{(k)}}\left(\bgamma(\bh)\right)} \right) \right) \right] \right| \\
		\leq& 2\Delta,
		\end{aligned}
		\end{equation}
		
		where $ \widetilde{D}_v(\cdot) = 1 - D_v(\cdot) $. In the above derivation, the first equality uses the definition of the neural net distance. The last inequality holds by the assumption that $ \phi(\cdot) $ is $ L_{\phi} $-Lipschitz and bounded in $ [-\Delta, \Delta] $.
		
		By summing over $ k $, and consider any fixed $ G_u \in \mG $, we obtain:
		\begin{align*}
		\sum_{k=1}^{n+1} d_{\mF, \phi} \left( \widehat{\mD}_{ {G_{\widehat{w}^{(k)}, \mH^{(k)}}\left(\bgamma(\bh)\right)} }, \widehat{\mD}_{\real} \right)	
		\leq& \sum_{k=1}^{n+1} d_{\mF, \phi} \left(\widehat{\mD}_{ {G_{\widetilde{w}, \mH^{(k)}}\left(\bgamma(\bh)\right)} }, \widehat{\mD}_{\real} \right) + 2(n+1) \Delta \\
		\leq& \sum_{\bh \in \mH^{(k)}, k=1}^{n+1} d_{\mF, \phi} \left( \widehat{\mD}_{\sum_{\bv \in \mC} \gamma_{\bv} \left(\bh \right) G_u (\bv)}, \widehat{\mD}_{\real} \right) + 2(n+1) \Delta \\
		\leq& \sum_{\bh \in \mH^{(k)}, k=1}^{n+1} d_{\mF, \phi} \left( \widehat{\mD}_{ G_{u} (\bh)}, \widehat{\mD}_{\real} \right) + \sum_{k=1}^{n+1} L_{\phi} Q_{L_{\bh}, L_{G}} (\bgamma, \mC) + 2(n+1) \Delta,
		\end{align*}
		where $ {Q}_{L_{\bh}, L_{G}} (\bgamma, \mC) = \mmE_{\bh} \left[ L_{\bh} \| \bh -  \br(\bh) \| + L_{G} \sum_{\bv \in \mC} |\gamma_{\bv} | \| \bv - \br(\bh) \|^2 \right] $. In the above derivation, the second
		inequality holds since $ \widetilde{w} $ is the minimizer of Problem (\ref{optimization}). The third inequality follows from the concavity of $ \phi(\cdot) $ and Lemma \ref{lemma1}:
		
		\begin{align*}
		d_{\mF, \phi} \left( {\mD}_{\sum_{\bv \in \mC, \bh \in \mH^{(k)}} \gamma_{\bv} \left(\bh \right) G_u (\bv)}, \widehat{\mD}_{\real} \right) =& \sup\limits \left| \mathop \mmE\nolimits_{\bx \in \widehat{\mD}_{\real}} \left[ \phi(D_v(\bx)) \right] + \mathop \mmE\nolimits_{\bh \in \mH^{(k)}} \left[ \phi\left(\widetilde{D}_v \left(\sum\nolimits_{\bv \in \mC} \gamma_{\bv} \left(\bh \right) G_u (\bv) \right)\right) \right] \right| \\
		\leq& \sup\limits \left| \mathop \mmE\nolimits_{\bx \in \widehat{\mD}_{\real}} \left[ \phi(D_v(\bx)) \right] + \mathop \mmE\nolimits_{\bh \in \mH^{(k)}} \left[ \phi\left(\widetilde{D}_v \left( G_u (\bh) \right) + \widehat{Q}_{L_{\bh}, L_{G}} (\bgamma, \mC) \right) \right] \right| \\
		\leq& \sup\limits \left| \mathop \mmE\nolimits_{\bx \in \widehat{\mD}_{\real}} \left[ \phi(D_v(\bx)) \right] + \mathop \mmE\nolimits_{\bh \in \mH^{(k)}} \left[ \phi\left(\widetilde{D}_v \left( G_u (\bh) \right) \right) \right] \right| + L_{\phi} Q_{L_{\bh}, L_{G}} (\bgamma, \mC) \\
		=& d_{\mF, \phi} \left( {\mD}_{G_u (\bh)}, \widehat{\mD}_{\real} \right) + L_{\phi} Q_{L_{\bh}, L_{G}} (\bgamma, \mC),
		\end{align*}
		where $ \widehat{Q}_{L_{\bh}, L_{G}} (\bgamma, \mC) = L_{\bh} \| \bh -  \br(\bh) \| + L_{G} \sum_{\bv \in \mC} |\gamma_{\bv} | \| \bv - \br(\bh) \|^2 $ and $ \mmE_{\bh} \left[ \widehat{Q}_{L_{\bh}, L_{G}} (\bgamma, \mC) \right] = Q_{L_{\bh}, L_{G}} (\bgamma, \mC) $. In the above derivation, the firth equality holds by the definition of the neural net distance. The first inequality because of Lemma \ref{lemma1} and the fact that $ \phi(\cdot) $ is a concave measuring function. Here, we suppose $ \phi(\cdot) $ is a monotonically increasing function. The second inequality holds by the following derivation:
		\begin{align*}
		&\left| \phi\left(\widetilde{D}_v \left( G_u (\bh) \right) + \widehat{Q}_{L_{\bh}, L_{G}} (\bgamma, \mC) \right) - \phi \left( \widetilde{D}_{v} \left( G_{u}(\bh) \right) \right) \right| \\
		\leq& \left| \phi' \left( \widetilde{D}_{v} \left( G_{u}(\bh) \right) \right) \left[ \left( \widetilde{D}_v \left( G_u (\bh) \right) + \widehat{Q}_{L_{\bh}, L_{G}} (\bgamma, \mC) \right) -  \widetilde{D}_{v} \left( G_{u}(\bh) \right) \right] \right| \\
		=& \left| \phi' \left( \widetilde{D}_{v} \left( G_{u}(\bh) \right) \right) \right| \widehat{Q}_{L_{\bh}, L_{G}} (\bgamma, \mC) \\
		\leq& L_{\phi} \widehat{Q}_{L_{\bh}, L_{G}} (\bgamma, \mC),
		\end{align*}
		In the above derivation, the first inequality uses the concavity of measuring function $ \phi(\cdot) $. The last inequality follows from that $ |\phi'| \leq L_{\phi} $. Now by taking expectation \wrt $ \mH_{r, n+1} $, we obtain
		\begin{align*}
		&\mmE_{\mH \subseteq \mH_{r, n+1}} \left[ d_{\mF, \phi} \left(\widehat{\mD}_{G_{\widehat{w}, \mH}\left( \bgamma(\bh) \right)}, \widehat{\mD}_{\real} \right) \right] \\
		\leq& \mmE_{\mH \subseteq \mH_{r, n+1}} \left[ d_{\mF, \phi} \left( \widehat{\mD}_{ G_{u, \bh \in \mH} (\bh)}, \widehat{\mD}_{\real} \right) \right] + L_{\phi} Q_{L_{\bh}, L_{G}} (\bgamma, \mC) + 2\Delta.
		\end{align*}
	\end{proof}

	\newpage
	\section{Proof of Theorem \ref{theorem: generalization_LCCGAN_v2}} \label{sec:thm2}
	\begin{thm} \label{theorem: generalization_LCCGAN_v2}
		Under the condition of Theorem \ref{theorem: supp_Generalization Bound}, and given an empirical distribution $ \widehat{\mD}_{\emph{real}} $ drawn from $ \mD_{\emph{real}} $, then the following holds with probability at least $ 1 - \delta $, 
		\begin{align}
		\left| \mmE_{\mH} \left[d_{\mF, \phi} \left(\widehat{\mD}_{G_{\widehat{w}}}, {\mD}_{\emph{real}} \right) \right]
		- \inf_{ \mG } \mmE_{\mH} \left[ d_{\mF, \phi} \left({\mD}_{G_u},  {\mD}_{\emph{real}} \right) \right] \right| 
		\leq 2 {R}_{\mX}(\mF) + 2 \Delta \sqrt{\frac{2}{N} \log(\frac{1}{\delta})} + 2\epsilon(d_{\mM}),
		\end{align}
		where $ {R}_{\mX}(\mF) = \mathop\mmE\limits_{\sigma, \mX} \left[ \sup\limits_{\mF} \frac{1}{N} \sum\limits_{i=1}^{N} \sigma_i \phi\left( D_v(\bx_i) \right) \right] $ and $ \sigma_i \in \{-1, 1\}, i = 1, 2, \ldots, m $ are independent uniform random variables.
	\end{thm}
	\begin{proof}
		For the real distribution $ \mD_{\real} $, we are interested in the generalization error in term of the following neural net distance:
		\begin{align}
		&\left| \mmE_{\mH} \left[d_{\mF, \phi} \left(\widehat{\mD}_{G_{\widehat{w}}}, {\mD}_{\real} \right) \right]
		- \inf_{ \mG } \mmE_{\mH} \left[ d_{\mF, \phi} \left({\mD}_{G_u},  {\mD}_{\real} \right) \right] \right|\nonumber \\
		\leq& \left| \mmE_{\mH} \left[ d_{\mF, \phi}  \left(\widehat{\mD}_{G_{\widehat{w}}}, {\mD}_{\real} \right) \right]
		- \mmE_{\mH} \left[ \inf_{\mG} d_{\mF, \phi} \left({\mD}_{G_u}, {\mD}_{\real} \right) \right] \right| \nonumber \\
		=& \left| \mmE_{\mH} \left[ d_{\mF, \phi} \left(\widehat{\mD}_{G_{\widehat{w}}},  {\mD}_{\real} \right)
		- d_{\mF, \phi} \left( \widehat{\mD}_{G_{\widehat{w}}}, \widehat{\mD}_{\real} \right)
		+ d_{\mF, \phi} \left( \widehat{\mD}_{G_{\widehat{w}}}, \widehat{\mD}_{\real} \right)
		- \inf_{\mG} d_{\mF, \phi} \left({\mD}_{G_u},  {\mD}_{\real} \right) \right] \right| \nonumber \\
		\leq& \left| \mmE_{\mH} \left[ d_{\mF, \phi} \left( \widehat{\mD}_{G_{\widehat{w}}}, {\mD}_{\real} \right)
		- d_{\mF, \phi} \left(\widehat{\mD}_{G_{\widehat{w}}}, \widehat{\mD}_{\real} \right)
		+ \inf_{\mG} d_{\mF, \phi} \left({\mD}_{G_u},  \widehat{\mD}_{\real} \right)
		- \inf_{\mG} d_{\mF, \phi} \left({\mD}_{G_u},  {\mD}_{\real} \right)
		+ \epsilon(d_{\mM}) \right] \right| \nonumber \\
		\leq& 2 \mmE_{\mH} \left[ \sup_{\mG} \left| d_{\mF, \phi} \left({\mD}_{G_u},  {\mD}_{\real} \right) - d_{\mF, \phi} \left({\mD}_{G_u},  \widehat{\mD}_{\real} \right) \right| + \epsilon(d_{\mM}) \right] \nonumber \\
		=& 2 \mmE_{\mH} \left[ \sup_{\mG} \left| \sup_{D_v \in \mF} \left| \mathop\mmE\limits_{\bx \in \mD_{\real}} \left[ \phi \left( D_v(\bx) \right) \right] + \mathop\mmE\limits_{\bx \in \mD_{G_u}} \left[ \phi \left( \widetilde{D}_v(\bx) \right) \right] \right|
		- \sup_{D_v \in \mF} \left| \mathop\mmE\limits_{\bx \in \widehat{\mD}_{\real}} \left[ \phi \left( D_v(\bx) \right) \right] + \mathop\mmE\limits_{\bx \in \mD_{G_u}} \left[ \phi \left( \widetilde{D}_v(\bx) \right) \right] \right| \right| + \epsilon(d_{\mM}) \right] \nonumber \\
		\leq& 2 \sup_{D_v \in \mF} \left| \mathop\mmE\limits_{\bx \in \mD_{\real}} \left[ \phi \left( D_v(\bx) \right) \right]
		- \mathop\mmE\limits_{\bx \in \widehat{\mD}_{\real}} \left[ \phi \left( D_v(\bx) \right) \right] \right| + 2\epsilon(d_{\mM}) \label{thm2: eq1}.
		\end{align}
		In the above derivation, the first inequality holds by by Jensen's inequality and the concavity of the infimum function.
		The second inequality holds by Theorem \ref{theorem: generalization_LCCGAN_v2}. The third inequality satisfies when we take supremum \wrt $ G_u \in \mG $. The last inequality uses the definition of the neural net distance and holds by triangle inequality. This reduces the problem to bounding the distance
		\begin{align}
		d'_{\mF} \left( \mD_{\real}, \widehat{\mD}_{\real} \right) := \sup_{D_v \in \mF} \left| \mmE_{\bx \in \mD_{\real}} \left[ \phi \left( D_v(\bx) \right) \right] - \mmE_{\bx \in \widehat{\mD}_{\real}} \left[ \phi \left( D_v(\bx) \right) \right] \right|, 
		\end{align}
		between the true distribution and its empirical distribution. This can be achieved by the uniform concentration bounds developed in statistical learning theory, and thus the distance $ d'_{\mF} \left( \mD_{\real}, \widehat{\mD}_{\real} \right) $ can be achieved by the Rademacher complexity.
		Let $ \bx_1, \bx_2, \ldots, \bx_N \in \mX $ be a set of $ N $ independent random samples in data space.  We introduce a function
		\begin{align}
		h\left( \bx_1, \bx_2, \ldots, \bx_N \right) = \sup_{D_v \in \mF} \left| \mmE_{\bx \in \mD_{\real}} \left[ \phi \left( D_v(\bx) \right) \right] - \mmE_{\bx \in \widehat{\mD}_{\real}} \left[ \phi \left( D_v(\bx) \right) \right] \right|.
		\end{align}
		Since measuring function $ \phi $ is Lipschitz and bounded in $ [-\Delta, \Delta] $, changing $ \bx_i $ to another independent sample $ \bx'_i $ can change the function $ h $ by no more than $ \frac{4\Delta}{ N } $, that is,
		\begin{align}
		h\left( \bx_1, \ldots, \bx_i \ldots, \bx_N \right) - h\left( \bx_1, \ldots, \bx'_i, \ldots, \bx_N \right) \leq \frac{4\Delta}{N},
		\end{align}
		for all $ i \in [1, N] $ and any points $ \bx_1, \ldots, \bx_N, \bx'_i \in \mX $.
		McDiarmid's inequality implies that with probability at least $ 1 - \delta $, the following inequality holds:
		\begin{align}
		&\sup_{D_v \in \mF} \left| \mmE_{\bx \in \mD_{\real}} \left[ \phi \left( D_v(\bx) \right) \right] - \mmE_{\bx \in \widehat{\mD}_{\real}} \left[ \phi \left( D_v(\bx) \right) \right] \right| \nonumber \\
		\leq & \mmE \left[ \sup_{D_v \in \mF} \left| \mmE_{\bx \in \mD_{\real}} \left[ \phi \left( D_v(\bx) \right) \right] - \mmE_{\bx \in \widehat{\mD}_{\real}} \left[ \phi \left( D_v(\bx) \right) \right] \right| \right]
		+ 2 \Delta \sqrt{\frac{2\log\left( \frac{1}{\delta} \right)}{N}}. \label{thm2: eq2}
		\end{align}
		From the bound on Rademacher complexity, we have
		\begin{align}
		\mmE \left[ \sup_{D_v \in \mF} \left| \mmE_{\bx \in \mD_{\real}} \left[ \phi \left( D_v(\bx) \right) \right] - \mmE_{\bx \in \widehat{\mD}_{\real}} \left[ \phi \left( D_v(\bx) \right) \right] \right| \right] 
		\leq 2 \mmE_{\sigma, \mX} \left[ \sup_{D_v \in \mF} \frac{1}{N} \sum_{i=1}^{N} \sigma_i \phi\left( D_v(\bx_i) \right) \right] = 2 R_{\mX} (\mF). \label{thm2: eq3}
		\end{align}
		Combining the inequalities (\ref{thm2: eq1}), (\ref{thm2: eq2}) and (\ref{thm2: eq3}), we have
		\begin{align*}
		& \mmE_{\mH} \left[d_{\mF, \phi} \left(\widehat{\mD}_{G_{\widehat{w}}}, {\mD}_{\real} \right) \right] - \inf_{G_u} \mmE_{\mH} \left[ d_{\mF, \phi} \left({\mD}_{G_u},  {\mD}_{\real} \right) \right] \leq 2 {R}_{\mX}(\mF) + 2 \Delta \sqrt{\frac{2 \log(\frac{1}{\delta})}{N}} + 2\epsilon(d_{\mM}).
		\end{align*}
	\end{proof}
	
	\newpage
	\section{Experimental Details} \label{sec:detail}
	\textbf{Implementation Details.}
	In the training, we follow the experimental settings in DCGAN~\cite{radford2015unsupervised}.
	Specifically, we use Adam~\cite{kingma2014adam} with a mini-batch size of 64 and a learning rate of 0.0002 to train the generator and the discriminator.
	Following the strategy in~\cite{he2015delving}, we initialize the parameters of both the generator and the discriminator.
	We set the hyperparameters $L_{\bh} {=} 1$ and $L_\nu {=} 0.0001$. All experiments are conducted on a single NVIDIA Titan X GPU.
	For all considered GAN methods, the inputs are sampled from a $d$-dimensional prior distribution, and we train the generative models to produce $64 \times 64$ images.
	
	Then, we introduce some details about StackGAN and Progressive GAN.
	For StackGAN, it is originally devised with an input text as the condition.  
	However, since there is no text data acting as the condition in our experiments, we remove its module of text embedding.
	For Progressive GAN, it is trained with a very large number of iterations and takes about 20 days for the training (reported in the original paper). However, the other GAN methods are only trained with a limited number of iterations to converge and take several hours for the training. In this sense, it is unfair to directly compare different GAN methods with different training settings.
	To address this, we train different GAN models with the same number of iterations to conduct a fair comparison.

	\section{More Results}\label{sec:more_res}
	In Table \ref{fig:supp_celeba-result}, by introducing LCC sampling into the training, LCCGAN and LCCGAN++ with a low input dimension $d{=}30$ produce promising face images with better quality and larger diversity than DCGAN and Progressive GAN with $d{=}100$. 
	Moreover, given the same input dimension, our proposed LCCGAN++ shows better performance than LCCGAN-v1 and other baseline methods. 
	
	\begin{table}[h]
		\centering
		{
			\caption{		
				% Performance comparisons of LCCGAN++ with several baselines with different $d$ on CelebA.
				Comparison of GAN methods with different dimensions on CelebA.
			}
			\label{fig:supp_celeba-result}
			\includegraphics[width=0.9\linewidth]{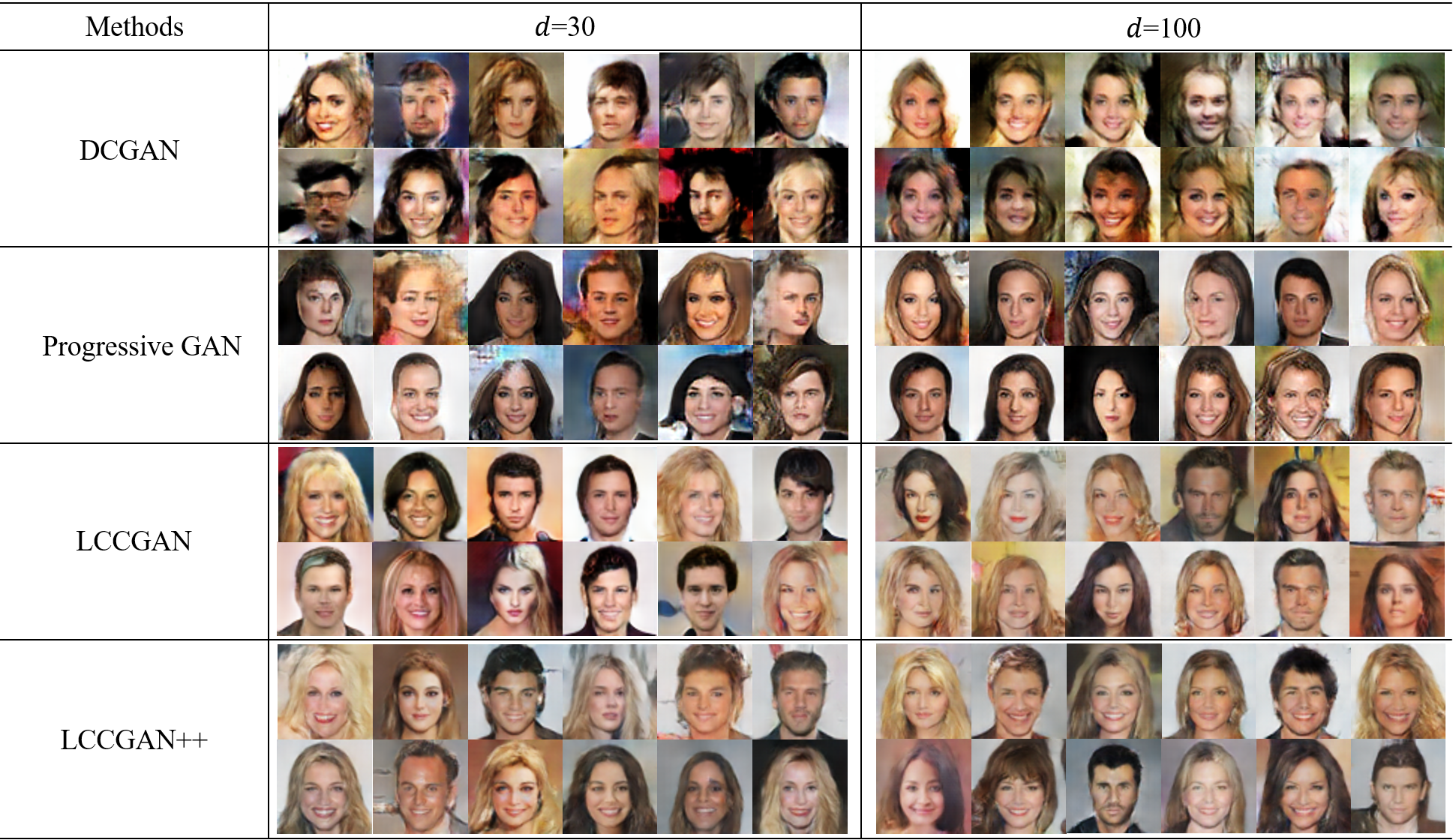}
		}
	\end{table}
	
\end{document}